\documentclass[twoside,11pt]{article}

\usepackage{jmlr2e}

\usepackage{lastpage}
\ShortHeadings{Krylov Subspace Method for NDS with Random Noise}{Hashimoto, Ishikawa, Ikeda, Matsuo and Kawahara}

\firstpageno{1}

\usepackage{url}            
\usepackage{amsfonts}       
\usepackage{amsmath,amssymb}
\usepackage{nccmath}
\usepackage{mathrsfs}
\usepackage{algorithm,algorithmicx}
\usepackage{algpseudocode}
\usepackage{ifthen}
\usepackage{graphicx}
\usepackage{subfigure}
\usepackage{tabularx}
\usepackage{color}

\makeatletter 
 \newcommand{\figcaption}[1]{\def\@captype{figure}\caption{#1}} 
 \renewcommand{\@biblabel}[1]{[#1]} 
\makeatother

\newtheorem{prop}[theorem]{Proposition}
\newtheorem{defin}[theorem]{Definition}

\def\mcl#1{\mathcal{#1}}
\def\blacket#1{\left\langle #1\right\rangle}
\def\hil{\mcl{H}}
\def\nn{\nonumber}
\def\opn{\operatorname}

\def\iu{\mathrm{i}}
\def\VVert{\vert\hspace{-1pt}\vert\hspace{-1pt}\vert}
\DeclareMathOperator*{\Dom}{\mcl{D}}
\def\blackets#1{\langle #1\rangle}

\def\bblacketg#1{\bigg\langle #1\bigg\rangle}

\begin{document}

\title{Krylov Subspace Method for Nonlinear Dynamical Systems with Random Noise}

\author{\name Yuka Hashimoto \email yuka.hashimoto.rw@hco.ntt.co.jp \\
       \addr NTT Network Technology Laboratories, NTT Corporation\\
       3-9-11, Midori-cho, Musashinoshi, Tokyo, 180-8585, Japan / \\
       Graduate School of Science and Technology, Keio University\\
       3-14-1, Hiyoshi, Kohoku, Yokohama, Kanagawa, 223-8522, Japan
       \AND
       \name Isao Ishikawa \email ishikawa.isao.zx@ehime-u.ac.jp\\
       \addr Faculty of Science, Ehime University\\
       2-5, Bunkyo-cho, Matsuyama, Ehime, 790-8577, Japan / \\
       Center for Advanced Intelligence Project, RIKEN\\
       1-4-1, Nihonbashi, Chuo-ku, Tokyo 103-0027, Japan
       \AND
       \name Masahiro Ikeda \email masahiro.ikeda@riken.jp \\
       \addr Center for Advanced Intelligence Project, RIKEN\\
       1-4-1, Nihonbashi, Chuo-ku, Tokyo 103-0027, Japan / \\
       Faculty of Science and Technology, Keio University\\
       3-14-1, Hiyoshi, Kohoku, Yokohama, Kanagawa, 223-8522, Japan
       \AND
       \name Yoichi Matsuo \email yoichi.matsuo.ex@hco.ntt.co.jp \\
       \addr NTT Network Technology Laboratories, NTT Corporation\\ 
       3-9-11, Midori-cho, Musashinoshi, Tokyo, 180-8585, Japan
       \AND 
       \name Yoshinobu Kawahara \email kawahara@imi.kyushu-u.ac.jp\\
       \addr Institute of Mathematics for Industry, Kyushu University\\
       744, Motooka, Nishi-ku, Fukuoka, 819-0395, Japan / \\
       Center for Advanced Intelligence Project, RIKEN\\
       1-4-1, Nihonbashi, Chuo-ku, Tokyo 103-0027, Japan
}

\editor{}

\maketitle

\begin{abstract}
Operator-theoretic analysis of nonlinear dynamical systems has attracted much attention in a variety of engineering and scientific fields, endowed with practical estimation methods using data such as dynamic mode decomposition. In this paper, we address a lifted representation of nonlinear dynamical systems with random noise based on transfer operators, and develop a novel Krylov subspace method for estimating the operators using finite data, with consideration of the unboundedness of operators. For this purpose, we first consider Perron-Frobenius operators with kernel-mean embeddings for such systems. We then extend the Arnoldi method, which is the most classical type of Kryov subspace methods, so that it can be applied to the current case. Meanwhile, the Arnoldi method requires the assumption that the operator is bounded, which is not necessarily satisfied for transfer operators on nonlinear systems. We accordingly develop the shift-invert Arnoldi method for Perron-Frobenius operators to avoid this problem. Also, we describe an approach of evaluating predictive accuracy by estimated operators on the basis of the maximum mean discrepancy, which is applicable, for example, to anomaly detection in complex systems. The empirical performance of our methods is investigated using synthetic and real-world healthcare data.
\end{abstract}



\section{Introduction}\label{secintro}
Analyzing nonlinear dynamical systems using data is one of the fundamental but still challenging problems in various engineering and scientific fields. 
Recently, operator-theoretic analysis has attracted much attention for this purpose, with which the behavior of a nonlinear dynamical system is analyzed through representations with transfer operators such as Koopman operators and their adjoint ones, Perron-Frobenius operators \citep{mezic12,kawahara16}. 
Since transfer operators are linear even if the corresponding dynamical systems are nonlinear, 
we can apply sophisticated theoretical results and useful tools of the operator theory, and 
access the properties of dynamics more easily from both theoretical and practical viewpoints. 
This is one of the main advantages of using transfer operators compared with other methods for learning dynamical systems such as using recurrent neural networks (RNNs) and hidden Markov models.
For example, one could consider modal decomposition of nonlinear dynamics by using the spectral analysis in operator theory, which provides the global characteristics of the dynamics and is useful in understanding complex phenomena \citep{kutz13}. This topic has also been recently discussed in machine learning \citep{kawahara16,lusch17,takeishi17-2}.

However, many of the existing works mentioned above are on {\em deterministic} dynamical systems.
Quite recently, the extension of these works to random systems has been addressed in a few works.
The methods for analyzing deterministic systems with transfer operators are extended to cases in which dynamical systems are random \citep{mezic17,takeishi17}.
Also, the transfer operator for a stochastic process in reproducing kernel Hilbert spaces (RKHSs) is defined \citep{klus17}, which provides an approach of analyzing dynamics of random variables in RKHSs.

In this paper, we address a lifted representation of nonlinear dynamical systems with {\em random noise} based on transfer operators, and develop a novel Krylov subspace method for estimating the operator using finite data, with consideration of the unboundedness of operators.
To this end, we first consider Perron-Frobenius operators with kernel-mean embeddings for such systems.
We then extend the Arnoldi method, which is the most classical type of Krylov subspace methods, so that it can be applied to the current case. 
However, although transfer operators on nonlinear systems are not necessarily bounded, the Arnoldi method requires the assumption on the boundedness of an operator. 
We accordingly develop the shift-invert Arnoldi method for the Perron-Frobenius operators to avoid this problem.
Moreover, we consider an approach of evaluating the predictive accuracy with estimated operators on the basis of the maximum mean discrepancy (MMD), which is applicable, for example, to anomaly detection in complex systems. 
Finally, we investigate the empirical performance of our methods using synthetic data and also apply those to anomaly detection with real-world healthcare data.

The remainder of this paper is organized as follows. First, in Section~\ref{sec:background}, we review transfer operators 
and Krylov subspace methods. 
In Section~\ref{sec:pf_kme}, we consider Perron-Frobenius operators with kernel-mean embeddings for nonlinear dynamical systems with random noises. 
In Section~\ref{sec:krylov}, we develop Krylov subspace methods for estimating these operators using data, 
and in Section~\ref{sec:related}, we discuss the connection of our methods to existing methods.
In Section~\ref{secanormal}, we consider an approach of evaluating the prediction accuracy with estimated operators. 
Finally, we empirically investigate the performance of our methods in Section~\ref{sec:result} and conclude the paper in Section~\ref{sec:conclusion}. 
Proofs which are not given after their statements are given in Appendix~\ref{ap:pf_kme}.

\paragraph{Notations}
Standard capital letters and ornamental capital letters denote the infinite dimensional linear operators.
Bold letters denote the matrices (finite dimensional linear operators) or finite dimensional vectors.
Calligraphic capital letters and italicized Greek capital letters denote sets.
The inner product and norm in $\hil_k$ are denoted as $\blacket{\cdot,\cdot}_k$ and $\Vert \cdot\Vert_k$, respectively.
The operator norm of a bounded linear operator $A$ in $\hil_k$, which is defined as $\sup_{v\in\hil_k,\Vert v\Vert_k=1}\Vert Av\Vert_k$ is denoted as $\VVert A\VVert_k$.
While, the Euclid norm in $\mathbb{C}^S$ for $S\in\mathbb{N}$ is denoted as $\Vert\cdot\Vert$, and the operator norm of a matrix $\mathbf{A}$ is denoted as $\VVert \mathbf{A}\VVert$.

The typical notations in this paper are listed in Table~\ref{tab1}.

\begin{table}[t]
\caption{Notation table}\label{tab1}\vspace{.3cm}
 \begin{tabularx}{\textwidth}{|c|X|}
\hline
$(\varOmega,\mcl{F},P)$  & A measurable space (sample space) with a probability measure $P$\\
$(\mcl{X},\mcl{B})$ & A Borel measurable and locally compact Hausdorff vector space (state space)\\
$x_t$ & A random variable from $\varOmega$ to $\mcl{X}$ represents the observation at $t$\\
$\{\xi_t\}$ & An i.i.d.\@ stochastic process corresponds to the random noise, where $\xi_t:\varOmega\to\mcl{X}$\\
$k$ & A positive-definite continuous, bounded and $c_0$-universal kernel on $\mcl{X}$\\
$\phi$ & The feature map endowed with $k$\\
$\hil_k$ & The RKHS endowed with $k$\\
$\mcl{M}(\mcl{X})$ & The set of all finite complex-valued regular Borel measures on $\mcl{X}$\\
$\Phi$ & The kernel mean embedding $\mcl{M}(\mcl{X})\to\hil_k$ defined by $\mu\mapsto\int_{x\in\mcl{X}}\phi(x)\;d\mu(x)$\\
$K$ & A Perron-Frobenius operator\\
$\Dom(A)$ & The domain of a linear operator $A$\\
$\varLambda(A)$ & The spectrum of an $A$\\
$\mcl{W}(A)$ & The numerical range of an $A$ on $\hil$ defined by $\{\blacket{Av,v}\mid\ v\in\hil,\ \Vert v\Vert=1\}$\\
$s_{\opn{min}}(\mathbf{A})$ & The minimal singular value of a matrix $\mathbf{A}$ defined by $\min_{\Vert \mathbf{w}\Vert_=1}\Vert \mathbf{Aw}\Vert$\\
$\gamma$ & A parameter to transform $K$ to a bounded bijective operator $(\gamma I-K)^{-1}$ which is not in $\varLambda(K)$\\
$\{\tilde{x}_0,\tilde{x}_1,\ldots\}$ & Observed time-series data\\
$S$ & A natural number that represents the dimension of the Krylov subspace\\
$N$ & A natural number that represents the amount of observed data used for the estimation\\
$\mu_{t,N}$ & The empirical measure generated by finite observed data $\{\tilde{x}_t,\ldots,\tilde{x}_{t+(S+1)(N-1)}\}$\\
$\mu_{t}$ & The weak limit of $\mu_{t,S}$ in $\mcl{M}(\mcl{X})$\\
$\mcl{V}(A,v)$ & The Krylov subspace of a linear operator $A$ and a vector $v$\\
$Q_S$ & The linear operator from $\mathbb{C}^{S}$ to $\hil_k$ composed of the orthonormal basis of the Krylov subspace\\
$\mathbf{R}_S$ & The $S$ times $S$ matrix which transforms the coordinate into the one with the orthonormal basis\\
$\tilde{\mathbf{K}}_S$ & The estimation of $K$ in an $S$-dimensional Krylov subspace\\
$\tilde{\mathbf{L}}_S$ & The estimation of $(\gamma I-K)^{-1}$ in an $S$-dimensional Krylov subspace\\
$a_{t,S}$ & The abnormality at $t$ computed with $\tilde{\mathbf{K}}_S$\\
\hline
 \end{tabularx}
\end{table}
\section{Background}
\label{sec:background}

\subsection{Transfer operators}
\label{transfer_review}


Consider a {\em deterministic} dynamical system
\begin{equation*}
x_{t+1}=h(x_t),
\end{equation*}
where $h\colon \mcl{X}\to\mcl{X}$ is a map, $\mcl{X}$ is a state space and 
$x_t\in\mcl{X}$.
Then, the corresponding {\em Koopman operator}~\citep{koopman31}, which is denoted as $\mathscr{K}$, is a linear operator in some subspace $\mcl{M}\subseteq \{g\colon\mcl{X}\to\mcl{X}\}$, defined by
\begin{equation*}
\mathscr{K}g:=g\circ h    
\end{equation*}
for $g\in\mcl{M}$.
From the definition, $\mathscr{K}$ represents the time evolution of the system as $(\mathscr{K}^ng)(x_0)=g(h(\ldots h(x_0)))=g(x_n)$.
Since the Koopman operator is linear even when the dynamical system $h$ is nonlinear, the operator theory is valid for analyzing it. And, the adjoint of Koopman operator is called {\em Perron-Frobenius operator}.
The concept of the RKHS is combined with transfer operators, and Perron-Frobenius operators in an RKHS are addressed~\citep{kawahara16,ishikawa18}.
One of the advantages of using transfer operators in RKHSs is that they can describe dynamical systems defined in non-Euclidean spaces.
Let $\hil_k$ be the RKHS endowed with a positive definite kernel $k$, and let $\phi:\mathcal{X}\rightarrow\mathcal{H}_k$ be the feature map.
Then, the Perron-Frobenius operator in the RKHS for $h\colon\mcl{X}\to\mcl{X}$, which is denoted by $\mathscr{K}_{\opn{RKHS}}$, 
is a linear operator in $\hil_k$ defined by
\begin{equation*}
\mathscr{K}_{\opn{RKHS}}\phi(x):=\phi(h(x))
\end{equation*}
for $\phi(x)\in\opn{Span}\{\phi(x)\mid x\in\mcl{X}\}$.

Transfer operator has also been discussed for cases in which a dynamical system is {\em random}.
Let $(\mcl{X},\mcl{B},\mu)$ and $(\varOmega,\mcl{F},P)$ be probability spaces.
The following random system is considered~\citep{mezic17,takeishi17}:
\begin{equation*}
x_{t+1}=\pi(t,\omega,x_t),
\end{equation*}
where $\pi\colon \mathbb{Z}_{\ge 0}\times\varOmega\times \mcl{X}\to\mcl{X}$ is a map and $x_t\in\mcl{X}$.
Then, the Koopman operator, which is denoted as $\bar{\mathscr{K}}_t$, is a linear operator in $\mcl{L}^2(\mathcal{X})$ and defined as
\begin{equation*}
\bar{\mathscr{K}}_tg:=\int_{\omega\in\varOmega}g(\pi(t,\omega,\cdot))\;dP(\omega)
\end{equation*}
for $g\in\mcl{L}^2(\mcl{X})$.
Also, Perron-Frobenius operators in RKHSs for {\em a stochastic process} $\{x_t\}$ on $(\mcl{X},\mcl{B},\mu)$
whose probability density functions are $\{p_t\}$ are considered~\citep{klus17}.
The Perron-Frobenius operator in an RKHS $\hil_k$, which is denoted as $\bar{\mathscr{K}}_{\opn{RKHS},t}$, is a linear operator in $\hil_k$ and defined by 
\begin{equation*}
\bar{\mathscr{K}}_{\opn{RKHS},t}\mathscr{E}p_t:=\mathscr{U}p_t,
\end{equation*}
where $\mathscr{E}$ and $\mathscr{U}$ are respectively the embeddings of probability densities to $\hil_k$
defined as $q\mapsto\int_{x\in\mcl{X}}\phi(x)q(x)\;d\mu(x)$ and $q\mapsto\int_{x\in\mcl{X}}\int_{y\in\mcl{X}}\phi(y)p(y| x)q(x)\;d\mu(y)d\mu(x)$, and $p$ is a function satisfying $P(x_{t+1}\in A\mid \{x_t=x\})=\int_{y\in A}p(y| x)\;d\mu(y)$.

 The Koopman and Perron-Frobenius operators are {\em defined in infinite dimensional spaces and linear}, whereas original systems are {\em defined in finite dimensional spaces and nonlinear}.
The full nonlinear dynamics can be captured within the linear operator, which allows us to apply techniques for linear operators such as Krylov subspace methods and modal decomposition.
Meanwhile, since the operators are defined in infinite dimensional space, we need fine arguments with mathematics for constructing and analyzing algorithms related to these operators in general.

\subsection{Unbounded linear operators}
First, we review the definition of a linear operator in a Hilbert space $\hil$.
\begin{defin}
Let $\mathcal{S}$ be a dense subset of $\hil$.
A linear operator $A$ in $\hil$ is a linear map $A:\mathcal{S}\to\hil$.
The set $\mathcal{S}$, which is denoted as $\Dom(A)$, is called the domain of $A$.
If there exists $C>0$ such that the operator norm of $A$, which is defined as 
${\VVert A\VVert}:=\sup_{v\in\hil,\Vert v\Vert=1}\Vert Av\Vert$ is bounded by $C$,
then $A$ is called bounded.
\end{defin}
For a linear operator $A$, the spectrum and numerical range are defined as follows:
\begin{defin}
Let $\Gamma(A)$ be the set of ${\gamma}\in\mathbb{C}$ such that $({\gamma} I-A):\Dom(A)\to\hil$ is 
bijective and $(\gamma I-A)^{-1}$ is bounded.
The spectrum of $A$ is the set $\mathbb{C}\setminus\Gamma(A)$, which is denoted as $\varLambda(A)$.
Moreover, the numerical range of $A$ is the set $\{\blacket{Av,v}\in\mathbb{C}\mid v\in\Dom(A),\ \Vert v\Vert=1\}$, which is denoted as $\mcl{W}(A)$.
\end{defin}
If $A$ is bounded, it can be shown that $\varLambda(A)$ is nonempty and compact~\citep[Theorem 2.1, Theorem 2.2]{kubrusly12}.
Also, by Toeplitz-Hausdorff theorem, it can be shown that $\mcl{W}(A)$ is bounded and convex~\citep{mcintosh78}.
The relation between $\varLambda(A)$ and $\mcl{W}(A)$ is characterized by the inclusion $\varLambda(A)\subseteq\overline{\mcl{W}(A)}$.
However, if $A$ is unbounded, neither $\varLambda(A)$ nor $\mcl{W}(A)$ is always bounded.

\subsection{Krylov subspace methods}
\label{krylov_review}

Krylov subspace methods are numerical methods for estimating the behavior of a linear operator by projecting it onto a finite dimensional subspace, called Krylov subspace. 
Let $A$ be a linear operator in Hilbert space $\hil$ and $v\in\hil$.
Then, the Krylov subspace of $A$ and $v$, which is denoted by $\mcl{V}_S(A,v)$, is an $S$-dimensional subspace
\begin{equation*}
\opn{Span}\{v,Av,\ldots,A^{S-1}v\}.
\end{equation*}
Krylov subspace methods are often applied to compute the spectrum of $\mathbf{A}$, $\mathbf{A}^{-1}\mathbf{v}$, or $f(\mathbf{A})\mathbf{v}$ for a given large and sparse $N\times N$ matrix $\mathbf{A}$, vector $\mathbf{v}\in\mathbb{C}^N$ and function $f$~\citep{krylov31,hestens52,saad83,gallopoulos92,moret04}.
The theoretical extensions of Krylov subspace methods for linear operators in infinite dimensional Hilbert spaces are explored in~\cite{guttel10,grimm12,gockler14,hashimoto_jjiam} to deal with matrices that are finite dimensional approximations of infinite dimensional linear operators.

The Arnoldi method is a classical and most commonly-used Kryov subspace method.
With the Arnoldi method, the Krylov subspace $\mcl{V}_S(A,v)$ is first constructed, and $A$ is projected onto it.
For a matrix $\mathbf{A}$ and vector $\mathbf{v}$, since the basis of $\mcl{V}_S(\mathbf{A},\mathbf{v})$ can be computed only by matrix-vector products,
the projection of $\mathbf{A}$ is also obtained only with matrix-vector products.
Note that the computational cost of the matrix-vector product is less than or equal to $O(N^2)$, which is less computationally expensive than computing the spectrum of $\mathbf{A}$, $\mathbf{A}^{-1}$ or $f(\mathbf{A})$ directly.

On the other hand, $\mathbf{A}$ is often the matrix approximation of an unbounded $A$, that is, the spatial discretization of $A$.
Theoretically, if $\mathbf{A}$ is an unbounded operator, $\mathbf{A}^iv$ for $i=1,\ldots,S-1$ cannot always be defined, 
and practically, although $\mathbf{A}$ is a matrix (bounded), the performance of the Arnoldi method for $\mathbf{A}$ degrades
due to the unboundedness of the original $A$.
To overcome this issue, 
the shift-invert Arnoldi method, that constructs the Krylov subspace $\mcl{V}_S((\gamma I-A)^{-1},v)$,
where $\gamma$ is not in the spectrum of $A$, has been investigated.
Since $(\gamma I-A)^{-1}$ is bounded, $(\gamma I-A)^{-i}v$ for $i=1,\ldots,S-1$ is always defined.
Thus, the Krylov subspace $\mcl{V}_S((\gamma I-A)^{-1},v)$ can be constructed.
This improve the performance for matrix $\mathbf{A}$, which is a matrix approximation of unbounded $A$.

Moreover, 
the application of the Arnoldi method to estimating transfer operators has been discussed for the deterministic case $\mathscr{K}_{\opn{RKHS}}$~\citep{kawahara16} and for the random case $\bar{\mathscr{K}}_t$~\citep{mezic17}.
An advantage of the Krylov subspace methods for estimating transfer operators is that they require one time-series dataset embedded by one observable function or one feature map, which matches the case of using an RKHS.
Meanwhile, the largest difference between the Krylov subspace methods mentioned in the preceding paragraphs  and those for transfer operators is that the operator to be estimated is given beforehand or not.
That is, calculations appear in Krylov subspace methods for transfer operators need to be carried out without knowing the operators.

\section{Perron-Frobenius Operators with Kernel-Mean Embeddings}
\label{sec:pf_kme}

Consider the following discrete-time nonlinear dynamical systems with
random noise in $\mcl{X}$:
\begin{equation}
 x_{t+1}=h(x_t)+\xi_{t},\label{model}
\end{equation}
where $t\in \mathbb{Z}_{\ge 0}$, $(\varOmega,\mcl{F})$ is a measurable space (corresponding to a sample space), $(\mcl{X},\mcl{B})$ is a Borel measurable and locally compact Hausdorff vector space (corresponding to a state space), $x_t$ and $\xi_t$ are random variables from sample space $\varOmega$ to state space $\mcl{X}$, and $h\colon\mcl{X}\to\mcl{X}$ is a map which can be nonlinear.
Let $P$ be a probability measure on $\varOmega$.
Examples of locally compact Hausdorff space are $\mathbb{R}^d$ and Riemannian manifolds.
Assume that $\xi_t$ with $t\in \mathbb{Z}_{\ge 0}$ is an i.i.d.\@ stochastic process and is independent of $x_t$.
The $\xi_t(\omega)$ corresponds to random noise in $\mcl{X}$.
We consider an RKHS on $\mcl{X}$.
Let $k\colon \mcl{X}\times\mcl{X}\to\mathbb{C}$ be a 
positive-definite kernel on $\mcl{X}$, i.e., $k$ satisfies

\begin{enumerate}
    \item $k(x,y)=\overline{k(y,x)}$ for $x,y\in\mcl{X}$,
    \item $\sum_{i,j=1}^n c_i\overline{c_j}k(x_i,x_j)\ge 0$ for $n\in\mathbb{N}$, $c_i\in\mathbb{C}$, and $x_i\in\mcl{X}$.
\end{enumerate}
The corresponding feature map is denoted by $\phi$, which is defined as $\phi(x)=k(x,\cdot)$.
Let $\hil_{k,0}:=\opn{Span}\{\phi(x)\mid x\in\mcl{X}\}$ and 
$\blacket{\cdot,\cdot}_k$ be an inner product on $\hil_{k,0}$ defined as
\begin{equation*}
   \bblacketg{\sum_{i=1}^nc_i\phi(x_i),\sum_{j=1}^mc_j\phi(x_j)}_k=\sum_{i=1}^n\sum_{j=1}^mc_i\overline{c_j}k(x_i,x_j).
\end{equation*}
The completion of $\hil_{k,0}$ is called a reproducing kernel Hilbert space (RKHS), which is denoted as $\hil_k$. 
In this paper, we assume that $k$ is continuous, bounded and $c_0$-universal, i.e., $\phi(x)\in\mcl{C}_0(\mcl{X})$ for all $x\in\mcl{X}$ and $\hil_k$ is dense in $\mcl{C}_0(\mcl{X})$.
Here, $\mcl{C}_0(\mcl{X})$ is the space of all continuous functions vanish at infinity \citep{sriperumbudur11}.
For example, the Gaussian kernel $e^{-c\Vert x-y\Vert^2}$ and
the Laplacian kernel $e^{-c\Vert x-y\Vert_1}$ with $c>0$ for $x,y\in \mathcal{X}$ with $\mcl{X}=\mathbb{R}^d$ are continuous and bounded $c_0$-universal kernels.

Now, we consider the transformation of the random variables in dynamical system~\eqref{model} into probability measures to capture the time evolution of the system starting from several initial states.
That is, random variable $x$ is transformed into probability measure $x_*P$,
where $x_*P$ denotes the push forward measure of $P$ with respect to $x$, defined by
$x_*P(B)=P(x^{-1}(B))$ for $B\in\mcl{B}$.
This transformation replaces the nonlinear relation $h$ between $x_t$ and $x_{t+1}$ with a linear one between probability measures.
Concretely, let $\beta_t\colon\mcl{X}\times\varOmega\to\mcl{X}$ be a map defined by $(x,\omega)\mapsto h(x)+\xi_t(\omega)$.
Then, a linear map $\mu\mapsto {\beta_t}_*(\mu\otimes P)$ is considered for a probability measure $\mu$, instead of $h$.
Also, we embed the probability measures into Hilbert space $\hil_k$, which defines an inner product between probability measures, to apply the operator theory.
Referring to \cite{klus17}, this embedding is possible by the kernel mean embedding~\citep{kernelmean} as follows.
Let $\mcl{M}(\mcl{X})$ be the set of all finite {complex-valued regular} Borel measures on $\mcl{X}$.
Then, the kernel mean embedding $\Phi\colon\mcl{M}(\mcl{X})\to\hil_k$ is defined by 
$\mu\mapsto\int_{x\in\mcl{X}}\phi(x)\;d\mu(x)$.

As a result, the Perron-Frobenius operator for dynamical system~\eqref{model} is defined with $\beta_t$ and the kernel mean embedding $\Phi$ as follows:
\begin{defin}
 The Perron-Frobenius operator for the system~\eqref{model},
 $K:\Phi(\mcl{M}(\mcl{X}))\to \hil_{k}$, is defined as
 \begin{equation}
 K\Phi(\mu):=\Phi({\beta_t}_*(\mu\otimes P)).\label{def}
 \end{equation}
\end{defin}
That is, $K$ transfers the measure generated by $x_t$ to that by $x_{t+1}$.
In fact, the following lemma holds.
\begin{lemma}\label{le:transfer}
The relation $K\Phi({x_t}_*P)=\Phi({x_{t+1}}_*P)$ holds.
\end{lemma}

Before discussing the estimation of $K$, we here describe some basic properties of the kernel mean embedding $\Phi$ and $K$, which are summarized as follows:
\begin{lemma}
\label{le:kme}
The kernel mean embedding 
$\Phi\colon\mcl{M}(\mcl{X})\to\hil_k$ is a linear and continuous map.
\end{lemma}
\begin{lemma}
\label{le:pf-operator}
The Perron-Frobenius operator 
$K:\Phi(\mcl{M}(\mcl{X}))\to \hil_{k}$ does not depend on $t$, is well-defined and is a linear operator. 
\end{lemma}
Also, the following two propositions show the connections of $K$ to the existing operators (stated in Section \ref{transfer_review}). 
We have the following relations of $K$ with $\bar{\mathscr{K}}_{\opn{RKHS},t}$ 
and with $\bar{\mathscr{K}}_t$: 
\begin{prop}\label{prop10}
If the stochastic process $\{x_t\}$ considered in \cite{klus17} 
satisfies $x_{t+1}=h(x_t)+\xi_t$, then $\bar{\mathscr{K}}_{\opn{RKHS},t}$ does not depend on $t$ and the identity $\bar{\mathscr{K}}_{\opn{RKHS}}\mathscr{E}p_t=K\Phi({x_t}_*P)$ holds.
\end{prop}
\begin{prop}\label{prop4}
If the random dynamical system 
$\pi$ satisfies $\pi(t,\omega,x)=\beta_t(\omega,x)=h(x)+\xi_t(\omega)$, then the Koopman operator $\bar{\mathscr{K}}_t$ in $\mathcal{H}_k$ does not depend on $t$ and is the adjoint operator of ${K}$. 
\end{prop}


\section{Krylov Subspace Methods for Perron-Frobenius Operators in RKHSs}
\label{sec:krylov}

In this section, we describe the estimation problem of the Perron-Frobenius operator $K$ defined as Eq.~\eqref{def}.
For this purpose, 
we extend Krylov subspace methods 
to our case. 
We first extend the classical Arnoldi method to our case in Subsection~\ref{secarnoldi}.
Although this method requires $K$ to be bounded for its convergence, $K$ is not necessarily bounded even for standard situations. 
For example, if $k$ is the Gaussian kernel, $h$ is nonlinear and $\xi_t\equiv0$, then $K$ is unbounded~\citep{iispre}.
Therefore, we develop a novel shift-invert Arnoldi method to avoid this issue in Subsection~\ref{secsia}. 
Although these two subsections discuss the ideal situations with infinite length of time-series data, we consider practical situations with finite ones
in Subsection~\ref{secpractical}.

With both methods, we construct the basis of the Krylov subspace as follows. 
Let $S\in \mathbb{N}$ be the dimension of the Krylov subspace constructed using observed time-series data
$\{\tilde{x}_0,\tilde{x}_1,\ldots\}$, which is assumed to be generated by dynamical system \eqref{model} with sample $\omega_0\in\varOmega$.
To generate elements of a basis of the Krylov subspace in terms of kernel mean embedding of probability measures, we split the observed data into $S+1$ datasets as
$\{\tilde{x}_0,\tilde{x}_{S'},\ldots\}$, $\{\tilde{x}_1,\tilde{x}_{1+S'},\ldots\}$, $\ldots$, $\{\tilde{x}_{S},\tilde{x}_{S+S'},\ldots\}$, where $S'=S+1$.
Then we define each element of the basis as the time average of each subset above in the RKHS.
\subsection{Arnoldi method for bounded operators}\label{secarnoldi}

For $t=0,\dots, S$, let $\mu_{t,N}:=\frac{1}{N}\sum_{i=0}^{N-1}\delta_{\tilde{x}_{t+iS'}}$ be the empirical measure constructed from observed data,
where $\delta_x$ denotes the Dirc measure centered at $x\in\mcl{X}$, and
 $\Psi_{0,N}:=[\Phi(\mu_{0,N}),\ldots,\Phi(\mu_{S-1,N})]$ with $N\in \mathbb{N}$.
By the definition of $K$, the following relation holds:
\begin{equation}
K\Psi_{0,N}=\left[\Phi\left({\beta_0}_*\left(\mu_{0,N}\otimes P\right)\right),\ldots,
  \Phi\left({\beta_{S-1}}_*\left(\mu_{S-1,N}\otimes P\right)\right)\right].\label{eq3}
\end{equation}
The calculation on the right-hand side of the Eq.~\eqref{eq3} is possible only if $\beta_t$ is available.
However, in practical situations, $\beta_t$ is not available. Therefore, $\Phi({\beta_t}_*(\mu_{t,N}\otimes P))$ is not available either.
To avoid this problem, we 
assume the following condition, which is similar to ergodicity, i.e., for any measurable and integrable function $f$, the following identity holds:
 \begin{equation}
 \begin{split}
  &\lim_{N\to\infty}\frac1N\sum_{i=0}^{N-1}\int_{\omega\in\varOmega}f(h(\tilde{x}_{t+iS'})+\xi_t(\omega))\;dP(\omega)\\
  & \qquad=\lim_{N\to\infty}\frac1N\sum_{i=0}^{N-1}f(h(\tilde{x}_{t+iS'})+\xi_{t+iS'}(\eta))\ a.s.\ \eta\in\varOmega\quad (t=0,\ldots,S).\label{eq8}
  \end{split}
 \end{equation}
Here, while the left-hand side of assumption~\eqref{eq8} represents the space average of $\xi_t$, the right-hand side gives its time average.
As a result, 
$\lim_{N\to\infty}\Phi\left({\beta_t}_*\left(\mu_{t,N}\otimes P\right)\right)$ can be calculated without $\beta_t$, which is stated as follows:
\begin{prop}\label{le:ergodic}
Under assumption~\eqref{eq8}, the following identity holds for $t=0\ldots,S-1$:
\begin{align*}
&\lim_{N\to\infty}\Phi\left({\beta_t}_*\left(\mu_{t,N}\otimes P\right)\right)
=\lim_{N\to\infty}\Phi(\mu_{t+1,N}).
\end{align*}
\end{prop}
\begin{proof}
 By the definition of $K$, the identity $\lim_{N\to\infty}K\Phi(\mu_{t,N})=\lim_{N\to\infty}\Phi\left({\beta_t}_*(\mu_{t,N}\otimes P\right))$ holds.
Moreover, under assumption~\eqref{eq8}, the following equalities hold:
\begin{align*}
&\lim_{N\to\infty}\Phi\left({\beta_t}_*\left(\mu_{t,N}\otimes P\right)\right)
=\lim_{N\to\infty}\frac1N\sum_{i=0}^{N-1}\int_{\omega\in\varOmega}\phi(h(\tilde{x}_{t+iS'})+\xi_t(\omega))\;dP(\omega)\nn\\
&\qquad=\lim_{N\to\infty}\frac1N\sum_{i=0}^{N-1}\phi(h(\tilde{x}_{t+iS'})+\xi_{t+iS'}(\omega_0))
=\lim_{N\to\infty}\Phi\left(\frac1N\sum_{i=0}^{N-1}\delta_{\tilde{x}_{t+1+iS'}}\right)\\
&\qquad =\lim_{N\to\infty}\Phi(\mu_{t+1,N}),
\end{align*}
which completes the proof of the proposition.
\end{proof}

Assume $\mu_{t,N}$ converges weakly to a finite {complex-valued regular} measure $\mu_t$. 
Then, since $\Phi$ is continuous, $\lim_{N\to\infty}\Phi(\mu_{t,N})=\Phi(\mu_{t})$ holds.
Moreover, if $K$ is bounded, then $K\Phi(\mu_{t})=\lim_{N\to\infty}K\Phi(\mu_{t,N})$ holds.
By Lemma~\ref{le:ergodic}, the limit of the right-hand side of Eq.~\eqref{eq3} is represented without $\beta_t$ as $[\Phi(\mu_{1}),\ldots,\Phi(\mu_{S})]$.
In addition, that of the left-hand side becomes $K\left[\Phi(\mu_0),\ldots,\Phi(\mu_{S-1})\right]$.
As a result, we have:
 \begin{align}
  &K\left[\Phi(\mu_0),\ldots,\Phi(\mu_{S-1})\right]
=\left[\Phi(\mu_{1}),\ldots,\Phi(\mu_S)\right]. \label{eq4}
 \end{align}
Note that the range of the operator $[\Phi(\mu_0),\ldots,\Phi(\mu_{S-1})]$ in Eq.~\eqref{eq4} is the Krylov subspace $\mcl{V}_S(K,\Phi(\mu_0))$ (cf.\@ Section~\ref{krylov_review}) since
\begin{align*}
&\opn{Span}\{\Phi(\mu_0),\ldots,\Phi(\mu_{S-1})\}=\opn{Span}\{\Phi(\mu_0),K\Phi(\mu_0),\ldots,K^{S-1}\Phi(\mu_0)\}.
\end{align*}

Now, the estimation of $K$ is carried out as follows: First, define $\Psi_0$ and $\Psi_1$ as
\begin{equation*}
\Psi_{0}:=[\Phi(\mu_{0})\ldots,\Phi(\mu_{S-1})], ~\Psi_1:=[\Phi(\mu_{1})\ldots,\Phi(\mu_{S})].
\end{equation*}
Then, we orthogonally project $K$ to the Krylov subspace $\mcl{V}_S(K,\Phi(\mu_0))$ by QR decomposition. That is, let
\begin{equation*}
\Psi_{0}=Q_{S}\mathbf{R}_{S},
\end{equation*}
be the QR decomposition of $\Psi_{0}$, where $Q_S=[q_0,\ldots,q_{S-1}]$, $q_0,\ldots,q_{S-1}$ is an orthonormal basis of $\mcl{V}_S(K,\Phi(\mu_0))$, and $\mathbf{R}_{S}$ is an $S\times S$ matrix.
Note that since $(Q_SQ_S^*)^2=Q_SQ_S^*$ and $(Q_SQ_S)^*=Q_SQ_S^*$, $Q_SQ_S^*$ is an orthogonal projection, where $Q_S^*$ is the adjoint operator of $Q_S$.
Operator $Q_{S}$ transforms a vector in $\mathbb{C}^S$ into the corresponding vector in $\hil_k$, which is the linear combination of the orthonormal basis of $\mcl{V}_S(K,\Phi(\mu_0))$.
On the other hand, $Q_S^*$, the adjoint operator of $Q_S$, projects a vector in $\hil_k$ onto $\mathbb{C}^{S}$.
Moreover, $\mathbf{R}_{S}$ transforms the coordinate with basis $\{\Phi(\mu_0),\ldots,\Phi(\mu_{S-1})\}$ into that with $\{q_0,\ldots,q_{S-1}\}$.
By identifying the $S$-dimensional subspace $\mcl{V}_S(K,\Phi(\mu_0))$ with $\mathbb{C}^S$, 
a projection of $K$ onto $\mcl{V}_S(K,\Phi(\mu_0))$ is represented as an $S\times S$ matrix $Q_{S}^*KQ_{S}$.
This matrix gives a numerical approximation of $K$.
Let $\tilde{\mathbf{K}}_S:=Q_{S}^*KQ_{S}$.
As a result, since $Q_S=\Psi_0\mathbf{R}_S^{-1}$, the following equality is derived using Eq.~\eqref{eq4}:
\begin{align*}
 \tilde{\mathbf{K}}_S=Q_{S}^*\Psi_{1}\mathbf{R}_{S}^{-1},
\end{align*}
which shows that $\tilde{\mathbf{K}}_S$ can be calculated with only observed time series data $\{\tilde{x}_0,\tilde{x}_1,\ldots\}$.
We give a more detailed explanation of the QR decomposition for the current case and the pseudo-code of the above in Appendices~\ref{ap7} and \ref{ap8}, respectively.

Regarding the convergence of $\tilde{\mathbf{K}}_S$, we have the following proposition:
\begin{prop}\label{prop:conv_arnoldi}
Assume $K$ is bounded.
If the Krylov subspace $\mcl{V}_S(K,\Phi(\mu_0))$ converges to $\hil_k$, that is, if the projection $Q_SQ_S^*$ converges strongly to the identity map in $\hil_k$ as $S\to\infty$, 
then for $v\in\Phi(\mcl{M}(\mcl{X}))$, $Q_S\tilde{\mathbf{K}}_SQ_S^*v$ converges to $Kv$ as $S\to\infty$.
\end{prop}
\begin{proof}
Since $\tilde{\mathbf{K}}_S=Q_S^*KQ_S$ and $Kv\in\hil_k$, the following inequality holds:
\begin{align}
\Vert Q_S\tilde{\mathbf{K}}_SQ_S^*v-Kv\Vert_k
&\le \Vert Q_SQ_S^*KQ_SQ_S^*v-Q_SQ_S^*Kv\Vert_k+\Vert Q_SQ_S^*Kv-Kv\Vert_k\nn\\
&\le \VVert Q_SQ_S^*K\VVert_k \Vert Q_SQ_S^*v-v\Vert_k+\Vert Q_SQ_S^*Kv-Kv\Vert_k\nn\\
&\le \VVert K\VVert_k \Vert Q_SQ_S^*v-v\Vert_k+\Vert Q_SQ_S^*Kv-Kv\Vert_k\nn\\
&\to 0,\label{eq:conv_arnoldi}
\end{align}
as $S\to\infty$, which completes the proof of the proposition.
\end{proof}
Note that $\VVert K\VVert_k$ in Eq.~\eqref{eq:conv_arnoldi} is not always finite if $K$ is unbounded.
Thus, Proposition~\ref{prop:conv_arnoldi} is not always true if $K$ is unbounded.

{
 In practice, we can iteratively compute the Arnoldi or shift-invert Arnoldi (which will be proposed in the next subsection) approximations for $S=1,2,\ldots$ and stop the iteration after the discrepancy between the approximation at $S$ and $S-1$ becomes sufficiently small.}

\subsection{Shift-invert Arnoldi method for unbounded operators}\label{secsia}
The estimation of $K$ with the Arnoldi method does not always converge to $K$ if $K$ is unbounded.
Therefore, in this section, we develop the shift-invert Arnoldi method for estimating $K$ to avoid this issue.
With this method, we fix $\gamma\notin\varLambda(K)$ and consider a bounded bijective operator $(\gamma I-K)^{-1}:\hil_k\to\Phi(\mcl{M}(\mcl{X}))$, where $\varLambda(K)$ is the spectrum of $K$ under the assumption $\varLambda(K)\neq\mathbb{C}$.
And, bounded operator $(\gamma I-K)^{-1}$ instead of $K$ is projected onto a Krylov subspace.

For the projection of $(\gamma I-K)^{-1}$, we need to calculate the Krylov subspace of $(\gamma I-K)^{-1}$. 
However, since $K$ is unknown in the current case, directly calculating $(\gamma I-K)^{-i}$ thus, the Krylov subspace is intractable. Therefore, we construct the Krylov subspace using only data by setting a vector $w_S\in\hil_k$, which depends on the dimension of the Krylov subspace $S$, and computing $(\gamma I-K)^{-1}w_S$. 
The following proposition guarantees a similar identity to Eq.~\eqref{eq4}:
\begin{prop}\label{prop:sia}
Define $w_j:=\sum_{t=0}^{j}\binom{j}{t}(-1)^{t}\gamma^{j-t}\Phi(\mu_{t})$.
Then, we have
 \begin{align}
  &(\gamma I-K)^{-1}\left[w_1,\ldots,w_{S}\right]=\left[w_0,\dots,w_{S-1}
  \right].\nn
 \end{align}
Moreover, space
$\opn{Span}\{w_1,\dots,w_S\}$
is the Krylov subspace $\mcl{V}_S\left((\gamma I-K)^{-1}, w_S\right)$.
\end{prop}
\begin{proof}
 Based on Proposition~\ref{le:ergodic}, we have:
\begin{align}
  &\lim_{N\to\infty}(\gamma\Phi(\mu_{t,N})-\Phi(\mu_{t+1,N}))
  =\lim_{N\to\infty} \left(\gamma\Phi(\mu_{t,N})-\Phi\left(\left({\beta_t}_*\left(\mu_{t,N}\otimes P\right)\right)\right)\right)\nn \\
  &\quad=\lim_{N\to\infty}\left(\gamma\Phi(\mu_{t,N})-K\Phi(\mu_{t,N})\right)
 =\lim_{N\to\infty}(\gamma I-K)\Phi(\mu_{t,N}).\label{eq21}
 \end{align}
 Since $(\gamma I-K)^{-1}$ is bounded,
 applying $(\gamma I-K)^{-1}$ to both sides of Eq.~\eqref{eq21} derives the identity 
 $(\gamma I-K)^{-1} (\gamma\Phi(\mu_{t})-\Phi(\mu_{t+1}))=\Phi(\mu_{t})$.
 Thus, for $j=0,\ldots,S-1$, the following identity holds:
 \begin{align}
 &(\gamma I-K)^{-1}\sum_{t=0}^j\binom{j}{t}(-1)^{t}\gamma^{j-t}(\gamma\Phi(\mu_{t})-\Phi(\mu_{t+1}))
 =\sum_{t=0}^j\binom{j}{t}(-1)^{t}\gamma^{j-t}\Phi(\mu_{t}).\label{eq30}
 \end{align}
Since $\binom{j}{t}+\binom{j}{t-1}=\binom{j+1}{t}$, the following identities also hold:
 \begin{align}
 &\sum_{t=0}^{j}\binom{j}{t}(-1)^{t}\gamma^{j-t}(\gamma\Phi(\mu_{t})-\Phi(\mu_{t+1}))\nn\\
 &\quad=\sum_{t=0}^{j}\binom{j}{t}(-1)^{t}\gamma^{j+1-t}\Phi(\mu_{t})+\sum_{t=0}^{j}\binom{j}{t}(-1)^{t+1}\gamma^{j-t}\Phi(\mu_{t+1})\nn\\
 &\quad=\binom{j}{0}\gamma^{j+1}\Phi(\mu_0)
+\sum_{t=1}^{j}\left(\binom{j}{t}+\binom{j}{t-1}\right)(-1)^{t}\gamma^{j+1-t}\Phi(\mu_t)
+\binom{j}{j}(-1)^{j+1}\gamma^0\Phi(\mu_{j+1})\nn\\
 &\quad= \sum_{t=0}^{j+1}\binom{j+1}{t}(-1)^{t}\gamma^{j+1-t}\Phi(\mu_{t}).\label{eq35}
 \end{align}
  Since $w_j=\sum_{t=0}^{j}\binom{j}{t}(-1)^{t}\gamma^{j-t}\Phi(\mu_{t})$, by Eqs.~\eqref{eq30} and \eqref{eq35}, the identity $(\gamma I-K)^{-1}w_{j+1}=w_j$ holds.
 Thus the following identity holds:
  \begin{align}
  &(\gamma I-K)^{-1}\left[w_1,\ldots,w_{S}\right]=\left[w_0,\dots,w_{S-1}
  \right],\nn
 \end{align}
and space
$\opn{Span}\{w_1,\dots,w_S\}$
is the Krylov subspace $\mcl{V}_S((\gamma I-K)^{-1}, w_S)$.
\end{proof}
Note that $w_j$ can be calculated using only data.

We now describe the estimation procedure. First, define $\Psi_0$ and $\Psi_1$ as
\begin{equation*}
\Psi_0:=[w_1,\ldots,w_{S}],~\Psi_1:=[w_0,\ldots,w_{S-1}],
\end{equation*}
respectively. 
And, let $\Psi_{0}=Q_{S}\mathbf{R}_{S}$ be the QR decomposition of
$\Psi_{0}$.
Similar to the Arnoldi method, the projection of $(\gamma I-K)^{-1}$ to $\mcl{V}_S((\gamma I-K)^{-1},w_S)$ is formulated as
\begin{equation*}
\tilde{\mathbf{L}}_{S}:=Q_{S}^*\Psi_{1}\mathbf{R}_{S}^{-1},
\end{equation*}
by using Proposition~\ref{prop:sia}.
As a result, $K$ is estimated by transforming the projected $(\gamma I-K)^{-1}$ back into $K$ as 
\begin{equation*}
\tilde{\mathbf{K}}_{S}:=\gamma \mathbf{I}-\tilde{\mathbf{L}}_{S}^{-1}.
\end{equation*}
A more detailed explanation of the QR decomposition for the current case and the pseudo-code are found in Appendices~\ref{ap7} and \ref{ap8}, respectively.

Regarding the convergence of $\tilde{\mathbf{K}}_S$, we have the following proposition:
\begin{prop}\label{prop:conv_sia}
If the Krylov subspace $\mcl{V}_S((\gamma I-K)^{-1},w_S)$ converges to $\hil_k$, that is, if the projection $Q_SQ_S^*$ converges strongly to the identity map in $\hil_k$ as $S\to\infty$, 
and 
if the convergence of $Q_SQ_S^*$ to the identity map is faster than the increase in $\VVert Q_S\tilde{\mathbf{K}}_SQ_S^*(\gamma I-K)^{-1}\VVert_k$ along $S$, i.e., $\VVert Q_S\tilde{\mathbf{K}}_SQ_S^*(\gamma I-K)^{-1}\VVert_k\Vert Q_SQ_S^*u-u\Vert_k\to 0$ as $S\to\infty$ for arbitrary $u\in\hil_k$, 
then for $v\in\Phi(\mcl{M}(\mcl{X}))$, $Q_S\tilde{\mathbf{K}}_SQ_S^*v$ converges to $Kv$ as $S\to\infty$.
\end{prop}
\begin{proof}
Since $\tilde{\mathbf{K}}_S=\gamma\mathbf{I}-\tilde{\mathbf{L}}_S^{-1}$ and $\tilde{\mathbf{L}}_S=Q_S^*(\gamma I-K)^{-1}Q_S$, 
and since $v\in\Phi(\mcl{M}(\mcl{X}))$ can be represented as $v=(\gamma I-K)^{-1}u$ with some $u\in\hil_k$ by the bijectivity of $(\gamma I-K)^{-1}$, the following inequality holds:
\begin{align}
&\Vert Q_S\tilde{\mathbf{K}}_SQ_S^*v-Kv\Vert_k\nn\\
&\qquad\le \Vert Q_S\tilde{\mathbf{K}}_SQ_S^*(\gamma I-K)^{-1}u-Q_S\tilde{\mathbf{K}}_SQ_S^*(\gamma I-K)^{-1}Q_SQ_S^*u\Vert_k\nn\\
&\qquad\qquad\qquad+\Vert Q_Sg(\tilde{\mathbf{L}}_S)Q_S^*u-g((\gamma I-K)^{-1})u\Vert_k\nn\\
&\qquad\le \VVert Q_S\tilde{\mathbf{K}}_SQ_S^*(\gamma I-K)^{-1}\VVert_k \Vert Q_SQ_S^*u-u\Vert_k\nn\\
&\qquad\qquad\qquad+\vert \gamma\vert \Vert Q_S\tilde{\mathbf{L}}_SQ_S^*u-(\gamma I-K)^{-1}u\Vert_k+\Vert Q_SQ_S^*u-u\Vert_k,\label{eq:conv_sia}
\end{align}
where $g(z):=\gamma z-1$.
Since $u\in\hil_k=\Dom((\gamma I-K)^{-1})$, $Q_S\tilde{\mathbf{L}}_SQ_Su$ converges to $(\gamma I-K)^{-1}u$ in the same manner as Proposition~\ref{prop:conv_arnoldi}.
Also, under the assumption of $\VVert Q_S\tilde{\mathbf{K}}_SQ_S^*(\gamma I-K)^{-1}\VVert_k\Vert Q_SQ_S^*u-u\Vert_k\to 0$ as $S\to\infty$,
$\Vert Q_S\tilde{\mathbf{K}}_SQ_S^*v-Kv\Vert_k\to 0$ as $S\to\infty$.
\end{proof}
Note that since $Q_S\tilde{\mathbf{K}}_SQ_S^*$ and $(\gamma I-K)^{-1}$ are bounded, $\VVert Q_S\tilde{\mathbf{K}}_SQ_S^*(\gamma I-K)^{-1}\VVert_k$ in the first term of the last inequality of Eq.~\eqref{eq:conv_sia} is finite for a fixed $S$.
This situation is completely different from that of the Arnoldi method, in which case $\VVert Q_SQ_S^*K\VVert_k$ in Eq.~\eqref{eq:conv_arnoldi} cannot always be defined when $K$ is unbounded even if $S$ is fixed.
Concerning the second term of the last inequality of Eq.~\eqref{eq:conv_sia}, it represents the approximation error of bounded operator $(\gamma I-K)^{-1}$ by $\tilde{\mathbf{L}}_S$, which corresponds to the approximation error  
of $K$ by $\tilde{\mathbf{K}}_S$ with the Arnoldi method.

{
According to Proposition~\ref{prop:conv_sia}, for the convergence of the shift-invert Arnoldi method, we need the technical assumption concerning the convergence of the Krylov subspace.
For applications, $v$ is set as $v=\phi(\tilde{x}_t)$ for time $t$, where $t$ is greater than any $t'$ such that $\tilde{x}_{t'}$ is used for constructing the Krylov subspace. 
If the orbit of the dynamical system is periodic or approaching a fixed point, the distance between the Krylov subspace and $v$ quickly becomes small for sufficiently large $S$, since $\tilde{x}_{t'}$, observed data composing the Krylov subspace, approach $x_t$.
Thus, the distance between the Krylov subspace and $u=\gamma v-Kv$ in the proof also becomes small for sufficiently large $S$.
Therefore, the assumption is expected to be satisfied.
Unfortunately, showing the sufficient condition of the assumption theoretically is a challenging task.
We will empirically confirm the convergence of the shift-invert Arnoldi method in Subsection~\ref{secnum1}.}

\subsection{Computation with finite data}\label{secpractical}
In practice, $\mu_t$ are not available due to the finiteness of data $\{\tilde{x}_0,\tilde{x}_1,\ldots\}$.
Therefore, we need $\mu_{t,N}$ instead of $\mu_t$. We define $\Psi_{0,N}$ and $\Psi_{1,N}$ as the quantities that are obtained by replacing $\mu_t$ with $\mu_{t,N}$ in the definitions of $\Psi_0$ and $\Psi_1$. For example, we define $\Psi_{0,N}$ for the Arnoldi method (described in Subsection~\ref{secarnoldi}) by $\Psi_{0,N}:=[\Phi(\mu_{0,N}),\ldots,\Phi(\mu_{S-1,N})]$.
Also, we let $\Psi_{0,N}=Q_{S,N}\mathbf{R}_{S,N}$ be the QR decomposition of $\Psi_{0,N}$, and $\tilde{\mathbf{K}}_{S,N}$ be the estimator with $\Psi_{0,N}$ and $\Psi_{1,N}$ that corresponds to $\tilde{\mathbf{K}}_S$.

Then, we can show that the above matrices from finite data converge to the original approximators.
\begin{prop}\label{prop1}
As $N\rightarrow\infty$, the matrix $\tilde{\mathbf{K}}_{S,N}$ converges to matrix $\tilde{\mathbf{K}}_S$, and operator $Q_{S,N}:\mathbb{C}^S\to\hil_k$ converges to $Q_S$ strongly in $\mathcal{H}_k$.
\end{prop}
\begin{proof}
 The elements of $\mathbf{R}_{S,N}\in\mathbb{C}^{S\times S}$ and $Q_{S,N}^*\Psi_{1,N}\in\mathbb{C}^{S\times S}$ are composed
 of the finite linear combinations of the inner products between $\Phi(\mu_{t,N})$ in the RKHS.
 Since $\lim_{N\to\infty}\mu_{t,N}=\mu_t$ for each $t\in\{0,\ldots,S\}$,
 and since $\Phi$ is continuous, the identity $\lim_{N\to\infty}\Phi(\mu_{t,N})=\Phi(\mu_t)$ holds. 
 Therefore, by the continuity of the inner product $\blacket{\cdot,\cdot}_k$, $\blacket{\Phi(\mu_{t,N}),\Phi(\mu_{s,N})}_k$ converges to $\blacket{\Phi(\mu_t),\Phi(\mu_s)}_k$ for each $^{\forall}t,s\in\{0,\ldots,S\}$ as $N\to\infty$.
Thus, matrices $\mathbf{R}_{S,N}$ and $Q_{S,N}^*\Psi_{1,N}$ converge to $\mathbf{R}_S$ and $Q_S^*\Psi_1$ as $N\to\infty$, respectively.
 This implies matrix ${\tilde{\mathbf{K}}_{S,N}}$ converges to $\tilde{\mathbf{K}}_S$ as $N\to\infty$.

 Moreover, by the identity $\lim_{N\to\infty}\Phi(\mu_{t,N})=\Phi(\mu_t)$, $\lim_{N\to\infty}\Vert\Psi_{0,N}v-\Psi_{0}v\Vert_k=0$ holds for all $v\in\mathbb{C}^{S\times S}$.
 Since $k$ is bounded, for all $t\in\{0,\ldots,S\}$, there exist
 $\tilde{C}(t)>0$ such that
 $\Vert\Phi(\mu_{t,N})\Vert_k\le\tilde{C}(t)$ for all
 $N\in\mathbb{N}$.
 Thus, there exists $C>0$ such that
 $\VVert\Psi_{0,N}\VVert_k\le C$ for all
 $N\in\mathbb{N}$.
 Therefore, for all $v\in\mathbb{C}^{S}$, it is deduced that
 \begin{align*}
  &\Vert Q_{S,N}v-Q_Sv\Vert_k
  =\Vert\Psi_{0,N}\mathbf{R}_{S,N}^{-1}v-\Psi_{0}\mathbf{R}_S^{-1}v\Vert_k\nn\\
  &\qquad\le \Vert\Psi_{0,N}\mathbf{R}_{S,N}^{-1}v-\Psi_{0,N}\mathbf{R}_S^{-1}v\Vert_k
  +\Vert\Psi_{0,N}\mathbf{R}_S^{-1}v-\Psi_{0}\mathbf{R}_S^{-1}v\Vert_k\nn\\
  &\qquad\le\VVert\Psi_{0,N}\VVert_k\Vert
  \mathbf{R}_{S,N}^{-1}v-\mathbf{R}_S^{-1}v\Vert_k
  +\Vert\Psi_{0,N}\mathbf{R}_S^{-1}v-\Psi_{0}\mathbf{R}_S^{-1}v\Vert_k\nn\\
  &\qquad\le C\Vert
  \mathbf{R}_{S,N}^{-1}v-\mathbf{R}_S^{-1}v\Vert_k
  +\Vert\Psi_{0,N}\mathbf{R}_S^{-1}v-\Psi_{0}\mathbf{R}_S^{-1}v\Vert_k\nn\\
  &\qquad\to 0,
 \end{align*}
 as $N\to\infty$.
This implies that $Q_{S,N}$ converges to $Q_S$ strongly in $\mathcal{H}_k$, which completes the proof of the proposition.
\end{proof}
The convergence speeds of $Q_{S,N}\to Q_S$ and $\tilde{\mathbf{K}}_{S,N}\to\tilde{\mathbf{K}}_S$ depend on that of $\Phi(\mu_{t,N})\to\Phi(\mu_{t})$ as described in the proof of this proposition.
And, the following proposition gives the connection of the convergence of $\Phi(\mu_{t,N})\to\Phi(\mu_{t})$ with the property of noise $\xi_t$:
\begin{prop}\label{prop:convspeed}
 For all $\epsilon >0$ and for $t=1,\ldots,S$, if $N$ is sufficiently large, 
the probability of 
${\Vert \Phi(\mu_{t})-\Phi(\mu_{t,N})\Vert_k}\ge \epsilon$ is bounded by $4\sum_{i=0}^{N-1}\sigma_{t-1,i}^2/(N^2\epsilon^2)$ under the condition of $x_{t-1+iS'}=\tilde{x}_{t-1+iS'}$.
Here, $\sigma_{t,i}^2:=\int_{\omega\in\varOmega}\left\Vert \phi(h(\tilde{x}_{t+iS'})+\xi_{t+iS'}(\omega))-m_{t,i}\right\Vert_k^2\;dP(\omega)$ and $m_{t,i}:=\int_{\omega\in\varOmega}\phi(h(\tilde{x}_{t+iS'})+\xi_{t+iS'}(\omega))\;dP(\omega)$.
\end{prop}
\begin{proof}
Let $\bar{m}_{t,N}:=1/N\sum_{i=0}^{N-1}m_{t,i}$ and
$\bar{\sigma}_{t,N}^2:=\int_{\omega\in\varOmega}\Vert 1/N\sum_{i=0}^{N-1}\phi(h(\tilde{x}_{t+iS'})+\xi_{t+iS'}(\omega))-\bar{m}_{t,N}\Vert_k^2\;dP(\omega)$.
The following identities about $\bar{\sigma}_{t,N}$ hold:
\begin{align}
 \bar{\sigma}_{t,N}^2
 &=\int_{\omega\in\varOmega}\bigg\Vert \frac{1}{N}\sum_{i=0}^{N-1}(\phi(h(\tilde{x}_{t+iS'})+\xi_{t+iS'}(\omega))-m_{t,i})\bigg\Vert_k^2\;dP(\omega)\nn\\
&=\int_{\omega\in\varOmega}\frac{1}{N^2}\bigg(\sum_{i=0}^{N-1}\Vert \phi(h(\tilde{x}_{t+iS'})+\xi_{t+iS'}(\omega))-m_{t,i}\Vert_k^2\nn\\
&\quad+\sum_{\substack{i,j=0\\i\neq j}}^{N-1}\left\langle\phi(h(\tilde{x}_{t+iS'})+\xi_{t+iS'}(\omega))-m_{t,i}
,\phi(h(\tilde{x}_{t+jS'})+\xi_{t+jS'}(\omega))-m_{t,j}\right\rangle_k\bigg)\;dP(\omega)\nn\\
&=\frac{1}{N^2}\sum_{i=0}^{N-1}{\sigma_{t,i}^2}.\label{eq:variance}
\end{align}
The last equality holds because $\xi_{t+iS'}$ and $\xi_{t+jS'}$ are independent if $i\neq j$ and the following equality holds for $i,j\in\{0,\ldots,N-1\},\ i\neq j$ by definition of $m_{t,i}$: 
\begin{align*}
&\int_{\omega\in\varOmega}\left\langle\phi(h(\tilde{x}_{t+iS'})+\xi_{t+iS'}(\omega))-m_{t,i},\phi(h(\tilde{x}_{t+jS'})+\xi_{t+jS'}(\omega))-m_{t,j}\right\rangle_k\;dP(\omega)\\
&=\int_{\omega\in\varOmega}\int_{\eta\in\varOmega}\!\!\left\langle\phi(h(\tilde{x}_{t+iS'})+\xi_{t+iS'}(\omega))-m_{t,i},\phi(h(\tilde{x}_{t+jS'})+\xi_{t+jS'}(\eta))-m_{t,j}\right\rangle_kdP(\omega)dP(\eta)\\
&=0.
\end{align*}
Let $\epsilon>0$. By the Chebyshev's inequality and Eq.~\eqref{eq:variance}, it is derived that:
\begin{align*}
&P\left(\bigg\Vert \frac1N\sum_{i=0}^{N-1}\phi(h(\tilde{x}_{t-1+iS'})+\xi_{t-1+iS'}(\omega))-\bar{m}_{t-1,N}\bigg\Vert_k\ge \epsilon\right)\le\frac{\bar{\sigma}_{t-1,N}^2}{\epsilon^2}
=\frac{1}{N^2\epsilon^2}\sum_{i=0}^{N-1}\sigma_{t-1,i}^2.
\end{align*}
By assumption~\eqref{eq8}, for sufficiently large $N$, $P(\Vert {\bar{m}_{t-1,N}}-\lim_{N\to\infty}1/N\sum_{i=0}^{N-1}\phi(h({\tilde{x}_{t-1+iS'}})+{\xi_{t-1+iS'}}(\omega))\Vert_k\le \epsilon/2)=1$ holds.
Thus, we have:
\begin{align*}
&P(\Vert\Phi(\mu_{t,N})-\Phi(\mu_t)\Vert_k\ge\epsilon\mid {x_{t-1+iS'}=\tilde{x}_{t-1+iS'},}\ i=0,1,\ldots)\\
&=P\left(\bigg\Vert \frac1N\sum_{i=0}^{N-1}\phi(h({\tilde{x}_{t-1+iS'}})+{\xi_{t-1+iS'}}(\omega){)}-\!\lim_{N\to\infty}\!\frac1N\sum_{i=0}^{N-1}\phi(h({\tilde{x}_{t-1+iS'}})+{\xi_{t-1+iS'}}(\omega){)}\bigg\Vert_k\!\!\ge\epsilon\right)\\
&\le P\left(\bigg\Vert \frac1N\sum_{i=0}^{N-1}\phi(h({\tilde{x}_{t-1+iS'}})+{\xi_{t-1+iS'}}(\omega){)}-{\bar{m}_{t-1,N}}\bigg\Vert_k\right.\\
&\qquad\qquad\qquad\qquad+\left.\bigg\Vert {\bar{m}_{t-1,N}}-\lim_{N\to\infty}\frac1N\sum_{i=0}^{N-1}\phi(h({\tilde{x}_{t-1+iS'}})+{\xi_{t-1+iS'}}(\omega))\bigg\Vert_k\ge\epsilon\right)\\
&\le P\left(\bigg\Vert \frac1N\sum_{i=0}^{N-1}\phi(h({\tilde{x}_{t-1+iS'}})+{\xi_{t-1+iS'}}(\omega){)}-{\bar{m}_{t-1,N}}\bigg\Vert_k+\frac{\epsilon}{2}\ge\epsilon\right)
\le\frac{4}{N^2\epsilon^2}\sum_{i=0}^{N-1}\sigma_{t-1,i}^2,
\end{align*}
which completes the proof of the proposition.
\end{proof}
{
 The condition $x_{t-1+iS'}=\tilde{x}_{t-1+iS'}$ means that for each $t=1,\ldots,m$, we just focus on the noise related to the observables that construct $\Phi(\mu_{t,N})$.
Therefore, the proposition describes the probability of the deviation of $\Phi(\mu_{t,N})$ from the mean value caused by one time-step noise becoming larger than $\epsilon$.
}

\begin{remark}\label{rem:choice_N}
The value $\sigma_{t,i}$ represents the variance of $h(\tilde{x}_{t+iS'})+\xi_{t+iS'}(\omega)$ in RKHS and $\xi_t(\omega)$ is the only term that depends on $\omega$ in $h(\tilde{x}_{t+iS'})+\xi_{t+iS'}(\omega)$. 
Therefore, $\sigma_{t,i}$ is small if the variance of $\xi_t$ is small.
Thus, Proposition \ref{prop:convspeed} shows theoretically, if the variance of $\xi_t$ is small, the convergence is fast.
{Also, the proposition implies that if $\sigma_{t,i}\approx\sigma$ for some $\sigma>0$ for all $i=1,\ldots,N-1$ and if we set $N$ as $N\ge 4\sigma^2/(\delta\epsilon^2)$ for $\epsilon>0$ and $\delta>0$, then the probability of $\Vert \Phi(\mu_{t})-\Phi(\mu_{t,N})\Vert_k\ge \epsilon$ is bounded by $\delta$.
On the other hand, setting $N$ large may lead numerical instabilities as $S$ grows up, especially for the Arnoldi method.
This is because some pairs of $\Phi(\mu_{t,N})$, defined as averages of subsequences of observed data, may become approximately linearly dependent.
This phenomenon will be empirically confirmed in Subsection~\ref{secnum1}.}
\end{remark}


\section{Connection to Existing Methods}
\label{sec:related}

In the previous two sections, we defined a Perron-Frobenius operator for dynamical systems with random noise based on kernel mean embeddings and developed the Krylov subspace methods for estimating it. We now summarize the connection of the methods with the existing Krylov subspace methods for transfer operators on dynamical systems.

For $\mcl{L}^2(\mcl{X})$ space and {\em deterministic} dynamical systems, the Arnoldi method for the Krylov subspace with {\em Koopman} operator $\mathscr{K}$, 
\begin{equation*}
\mcl{V}_S(\mathscr{K},g)=\opn{Span}\{g,\mathscr{K}g,\ldots,\mathscr{K}^{S-1}g\},
\end{equation*}
is considered in \cite{kutz13}, where $g:\mcl{X}\to\mathbb{C}$ is an observable function.
Let $\{\tilde{z}_0,\ldots,\tilde{z}_{S-1}\}$ be the sequence generated from deterministic system $x_{t+1}=h(x_t)$.
Then, this Krylov subspace 
captures the time evolution starting from many initial values $\tilde{z}_0,\ldots,\tilde{z}_{S-1}$
by approximating $g$ with $[g(\tilde{z}_0),\ldots,g(\tilde{z}_{S-1})]$.
This idea is extended to the Krylov subspace with {\em Koopman} operator $\bar{\mathscr{K}}$, $\mcl{V}_S(\bar{\mathscr{K}},g)$, for the case in which the system is {\em random},
by assuming the following ergodicity \citep{mezic17,takeishi17}: for any measurable and integrable function $f$ with respect to a measure $\mu$,
\begin{equation}\label{takeishi_ergodicity}
\int_{x\in\mcl{X}}f(x)\;d\mu(x)=\lim_{N\to\infty}\frac1N\sum_{i=0}^{N-1}f(\tilde{z}_i).
\end{equation}
The following proposition states the connection with our assumption~\eqref{eq8} in Subsection~\ref{secarnoldi} and the assumption~\eqref{takeishi_ergodicity}. 
 \begin{prop}\label{prop9}
 For each $t=0,\ldots,S-1$, if there exists a random variable $y_t$ such that
 $\mu_t={y_t}_*P$, and $y_t$ is independent of $\xi_t$,
 then assumption \eqref{eq8} is equivalent to assumption~\eqref{takeishi_ergodicity} for $\mu={(h(y_t)+\xi_t)}_*P$ and $\tilde{z}_i=\tilde{x}_{t+1+iS'}$.
 \end{prop}

Meanwhile, Kawahara considers a {\em Perron-Frobenius} operator for {\em deterministic} systems in an RKHS $\hil_k$, and projects it to the following Krylov subspace~\citep{kawahara16}:
\begin{align}\label{kawahara_krylov}
&\mcl{V}_S(\mathscr{K}_{\opn{RKHS}},\phi(\tilde{z}_0))
=\opn{Span}\{\phi(\tilde{z}_0),\mathscr{K}_{\opn{RKHS}}\phi(\tilde{z}_0)\ldots,\mathscr{K}_{\opn{RKHS}}^{S-1}\phi(\tilde{z}_{0})\}.
\end{align}
Subspace~\eqref{kawahara_krylov} captures the time evolution starting from a single initial value $\tilde{z}_0$.
This prevents the straightforward extension of Krylov subspace \eqref{kawahara_krylov} to the subspace that is applicable to the case in which the dynamics is random. 
It can be shown that the Krylov subspace with {\em Perron-Frobenius} operator $K$, $\mcl{V}_S(K,\Phi(\mu_0))$, which is addressed in this paper for {\em random} systems, is a generalization of the Krylov subspace for the {\em deterministic} systems considered in~\cite{kawahara16}:
\begin{prop}\label{prop:krylov_extension}
The Krylov subspace $\mcl{V}_S(K,\Phi(\mu_0))$ generalizes the Krylov subspace introduced by Kawahara~\eqref{kawahara_krylov} to that for dynamical systems with random noise.
\end{prop}
Note that the framework of the Krylov subspace methods for {\em Perron-Frobenius} operators for {\em random} systems has not been addressed in prior works.
Also note that the theoretical analysis for these methods requires the assumption that the operator is bounded, which is not necessarily satisfied for transfer operators on discrete-time nonlinear systems \citep{iispre}.

The shift-invert Arnoldi method is a popular Krylov subspace method discussed in numerical linear algebra, which is applied to extract some information, for example, eigenvalues and a matrix function acting on a vector, from {\em given matrices}, 
and some theoretical analyses have been extended to given unbounded operators 
\citep{guttel10,grimm12,gockler14,hashimoto_jjiam}.
However, as far as we know, 
our paper is the first paper to address the unboundedness of Perron-Frobenius operators for the estimation problem and apply the shift-invert Arnorldi method to estimate Perron-Frobenius operators, which are {\em not known beforehand}. 
Since the shift-invert Arnoldi method was originally investigated for {\em given} matrices or operators, applying it to {\em unknown} operators is not straightforward, as described in Subsection~\ref{secsia}.

\section{Evaluation of Prediction Errors with Estimated Operators}
\label{secanormal}

In this section, we discuss an approach of evaluating the prediction accuracy with estimated Perron-Frobenius operators, which is applicable, for example, to anomaly detection in complex systems.

Consider the prediction of $\phi(\tilde{x}_{t})$ ($\in\hil_k$) using estimated operator $\tilde{\mathbf{K}}_{S,N}$ 
and observation (embedded in RKHS $\hil_k$) $\phi(\tilde{x}_{t-1})$.
This prediction is calculated as $Q_{S,N}\tilde{\mathbf{K}}_{S,N}Q_{S,N}^*\phi(\tilde{x}_{t-1})$.
Thus, the prediction error can be evaluated as
\begin{equation}
\Vert\phi(\tilde{x}_t)-Q_{S,N}\tilde{\mathbf{K}}_{S,N}Q_{S,N}^*\phi(\tilde{x}_{t-1})\Vert_k.\label{eq:mmd}
\end{equation}
Note that this is the maximum mean discrepancy (MMD) between $Q_{S,N}\tilde{\mathbf{K}}_{S,N}Q_{S,N}^*\phi(\tilde{x}_{t-1})$
and $\phi(\tilde{x}_t)$ in the unit disk in $\hil_k$~\citep{gretton12}.

For practical situations such as anomaly detection, we define the degree of abnormality for prediction at $t$ based on the MMD~\eqref{eq:mmd} as follows:
\begin{equation*}
 a_{t,S}:=\frac{\Vert\phi(\tilde{x}_t)-Q_{S}\tilde{\mathbf{K}}_{S}Q_{S}^*\phi(\tilde{x}_{t-1})\Vert_k}{\Vert Q_S\tilde{\mathbf{K}}_SQ_S^*\phi(\tilde{x}_{t-1})\Vert_k}.
\end{equation*}
The $a_{t,S}$ is bounded as follows:
\begin{align}
 &a_{t,S}=\frac{\Vert \phi(\tilde{x}_t) -Q_S\tilde{\mathbf{K}}_SQ_S^*\phi(\tilde{x}_{t-1})\Vert_k}{\Vert Q_S\tilde{\mathbf{K}}_SQ_S^*\phi(\tilde{x}_{t-1})\Vert_k}\nn\\
 &\qquad\le \frac{\Vert \phi(\tilde{x}_t) -K\phi(\tilde{x}_{t-1})\Vert_k}{\Vert Q_S\tilde{\mathbf{K}}_SQ_S^*\phi(\tilde{x}_{t-1})\Vert_k}
 +\frac{\Vert K\phi(\tilde{x}_{t-1})-Q_S\tilde{\mathbf{K}}_SQ_S^*\phi(\tilde{x}_{t-1})\Vert_k}{\Vert Q_S\tilde{\mathbf{K}}_SQ_S^*\phi(\tilde{x}_{t-1})\Vert_k}.\label{eq:a_dec}
\end{align}
Concerning the second term of the right-hand side in Eq.~\eqref{eq:a_dec}, the following proposition is derived directly by Propositions~\ref{prop:conv_arnoldi} and \ref{prop:conv_sia}:
\begin{prop}\label{prop:conv2}
Let $\tilde{\mathbf{K}}_S$ be the estimation using the shift-invert Arnoldi method.
Under the assumption of Proposition~\ref{prop:conv_sia},
${\Vert K\phi(\tilde{x}_{t-1})-Q_S\tilde{\mathbf{K}}_SQ_S^*\phi(\tilde{x}_{t-1})\Vert_k}/{\Vert Q_S\tilde{\mathbf{K}}_SQ_S^*\phi(\tilde{x}_{t-1})\Vert_k}$, the second term of the right-hand side in Eq.~\eqref{eq:a_dec}, converges to $0$ as $S\to\infty$.
For the Arnoldi method, the convergence is attained for the case in which $K$ is bounded and the Krylov subspace $\mcl{V}_S(K,\Phi(\mu_0))$ converges to $\hil_k$.
\end{prop}
On the other hand, the numerator of the first term of the right-hand side in Eq.~\eqref{eq:a_dec} 
represents the deviation of the observation $\tilde{x}_t$
from the prediction at $t$ under the assumption that 
$\tilde{x}_{t}$ is generated by the dynamical system~\eqref{model} in the RKHS,
because the identity $K\phi(\tilde{x}_{t-1})=\int_{\omega\in\varOmega}\phi(h(\tilde{x}_{t-1})+\xi_t(\omega))\;dP(\omega)$ holds by the definition of $K$.
And, the following proposition shows that the denominator indicates how vector $\phi(\tilde{x}_{t-1})$ deviates from the Krylov subspace:

\begin{prop}\label{prop3}
Let $\tilde{\mathbf{K}}_S$ be the estimation with the shift-invert Arnoldi method and 
let $g(z):=z/(\gamma z-1)$ for $z\in\mathbb{C}$.
If $g$ is holomorphic in the interior of $\mcl{W}((\gamma I-K)^{-1})$ and continuous in $\overline{\mcl{W}((\gamma I-K)^{-1})}$, 
then the following inequality 
holds:
\begin{align}
&\frac{1}{\Vert Q_S\tilde{\mathbf{K}}_SQ_S^*\phi(\tilde{x}_{t-1})\Vert_k} \le \frac{C}{\Vert Q_S^*\phi(\tilde{x}_{t-1})\Vert},\label{eq25}
\end{align}
where $C=(1+\sqrt{2})\sup_{z\in\mcl{W}((\gamma I-K)^{-1})}\vert g(z)\vert\ge 0$ is a constant.
For the Arnoldi method, inequality \eqref{eq25} is satisfied with $g(z)=1/z$ for the case in which $K$ is bounded and $g$ is holomorphic in the interior of $\mcl{W}(K)$ and continuous in $\overline{\mcl{W}(K)}$.
\end{prop}
To show Proposition \ref{prop3}, the following lemma by~\cite{cro17} is used.
\begin{lemma}\label{propcro}
Let $\mathbf{A}$ be a matrix. 
If $f$ is holomorphic in the interior of $\mcl{W}(\mathbf{A})$
and continuous in $\overline{\mcl{W}(\mathbf{A})}$, then there exists $0<C\le 1+\sqrt{2}$ such that
\begin{equation}
 \VVert f(\mathbf{A})\VVert\le C\sup_{z\in \mcl{W}(\mathbf{A})}\vert f(z)\vert.
\end{equation}
\end{lemma}
\begin{proof}(Proof of Proposition \ref{prop3})
Let $s_{\opn{min}}(\mathbf{A}):=\min_{\Vert \mathbf{w}\Vert_=1}\Vert \mathbf{Aw}\Vert$ be the minimal singular value of a matrix $\mathbf{A}$.
Since the relation $1/s_{\opn{min}}(\tilde{\mathbf{L}}_S)=\VVert \tilde{\mathbf{L}}_S^{-1}\VVert$ and inclusion $\mcl{W}(\tilde{\mathbf{L}}_S)\subseteq\mcl{W}((\gamma I-K)^{-1})$ hold, the following inequalities hold:
\begin{align}
 &\frac{1}{\Vert Q_S\tilde{\mathbf{K}}_SQ_S^*\phi(\tilde{x}_{t-1})\Vert_k}=\frac{1}{\Vert \tilde{\mathbf{K}}_SQ_S^*\phi(\tilde{x}_{t-1})\Vert}
\le\frac{1}{s_{\opn{min}}(\tilde{\mathbf{K}}_S)\Vert Q_S^*\phi(\tilde{x}_{t-1})\Vert}\nn\\
&\qquad\le \frac{\VVert \tilde{\mathbf{K}}_S^{-1}\VVert}{\Vert Q_S^*\phi(\tilde{x}_{t-1})\Vert}
\le \frac{(1+\sqrt{2})\sup_{z\in\mcl{W}(\tilde{\mathbf{L}}_S)}\vert g(z)\vert}{\Vert Q_S^*\phi(\tilde{x}_{t-1})\Vert}\nn\\
&\qquad \le \frac{(1+\sqrt{2})\sup_{z\in\mcl{W}((\gamma I-K)^{-1})}\vert g(z)\vert}{\Vert Q_S^*\phi(\tilde{x}_{t-1})\Vert}.\nn
\end{align}

For the Arnoldi method, if $K$ is bounded, the identity $Q_S^*KQ_S=\tilde{\mathbf{K}}_S$ holds.
Thus, the inequality $\VVert \tilde{\mathbf{K}}_S^{-1}\VVert\le(1+\sqrt{2})\sup_{z\in\mcl{W}(K)}\vert g(z)\vert$ with $g(z)=1/z$ holds in this case, 
which deduces the same result as inequality~\eqref{eq25}.
\end{proof}
If $\phi(\tilde{x}_{t-1})$ deviates from the Krylov subspace, the norm of the projected vector $Q_SQ_S^*\phi(\tilde{x}_{t-1})$, which is equal to $\Vert Q_S^*\phi(\tilde{x}_{t-1})\Vert$, becomes small.
Proposition~\ref{prop3} implies $a_{t,S}$ becomes large in this case.
On the other hand, if $\phi(\tilde{x}_t)$ is sufficiently close to the Krylov subspace, that is, $\min_{u\in\mcl{V}_S}\Vert \phi(\tilde{x}_{t-1})-u\Vert_k\approx 0$, then we have
\begin{align*}
 1&=\Vert \phi(\tilde{x}_{t-1})\Vert_k^2
 =\Vert \phi(\tilde{x}_{t-1})-Q_SQ_S^*\phi(\tilde{x}_{t-1})\Vert_k^2+\Vert Q_SQ_S^*\phi(\tilde{x}_{t-1})\Vert_k^2\\
 &=\min_{u\in\mcl{V}_S}\Vert \phi(\tilde{x}_{t-1})-u\Vert_k+\Vert Q_SQ_S^*\phi(\tilde{x}_{t-1})\Vert_k^2
 \approx \Vert Q_SQ_S^*\phi(\tilde{x}_{t-1})\Vert_k^2
=\Vert Q_S^*\phi(\tilde{x}_{t-1})\Vert^2, 
\end{align*}
if $k$ satisfies $k(x,x)=1$ for any $x\in\mcl{X}$, for example, the Gaussian and Laplacian kernels. 
As a result, if $\tilde{x}_{t}$ is generated by dynamical system~\eqref{model}, and 
if $\phi(\tilde{x}_{t-1})$ is sufficiently close to $\mcl{V}_S$, then
$a_{t,S}$ is bounded by a reasonable value.
Conversely, if $\tilde{x}_{t}$ is unlikely to be generated by dynamical system \eqref{model}, or 
$\phi(\tilde{x}_{t-1})$ is not close to the subspace $\mcl{V}_S$, then $a_{t,S}$ becomes large.
In the context of anomaly detection, since both the above cases mean $\tilde{x}_{t-1}$ or $\tilde{x}_{t}$ deviates from the regular pattern of times-series $\{\tilde{x}_0,\ldots,\tilde{x}_{T-1}\}$, they should be regarded as abnormal.

In practice, $Q_S$, $\mathbf{R}_S$, and $\tilde{\mathbf{K}}_S$
are approximated by $Q_{S,N}$, $\mathbf{R}_{S,N}$ and $\tilde{\mathbf{K}}_{S,N}$, respectively.
Thus, the following empirical value can be used:
\begin{equation*}
 a_{t,S,N}:=\frac{\Vert\phi(\tilde{x}_t)-
Q_{S,N}\tilde{\mathbf{K}}_{S,N}Q_{S,N}^*\phi(\tilde{x}_{t-1})\Vert_k}{\Vert Q_{S,N}\tilde{\mathbf{K}}_{S,N}Q_{S,N}^*\phi(\tilde{x}_{t-1})\Vert_k}.
\end{equation*}
By Proposition~\ref{prop1}, the following proposition about the convergence of $a_{t,S,N}$ holds:
\begin{prop}\label{prop:abnormality}
The $a_{t,S,N}$ converges to $a_{t,S}$ as $N\to\infty$.
\end{prop}

\section{Numerical Results}
\label{sec:result}

We empirically evaluate the behavior of the proposed Krylov subspace methods
in Subsection~\ref{secnum1}
then describe their
application to anomaly detection using real-world time-series data in
Subsection~\ref{secnum2}.
\subsection{Comparative Experiment}\label{secnum1}

The behavior of the Arnoldi and shift-invert Arnoldi methods (SIA in the figures) were evaluated
numerically based on the empirical abnormality.
We used 100 synthetic time-series datasets $\{\tilde{x}_0, \ldots, \tilde{x}_T\}$ randomly generated by the following {three dynamical systems}:
\begin{align}
&{x_0=2,}{\quad x_{t+1}=0.9995x_{t}+0.1\xi_t,\label{eqsyn0}}\\
&x_0=0.5,\quad x_{t+1}=0.99x_t\cos(0.1x_t)+\xi_t, \label{eqsyn}\\
&{x_0=0.1,}{\quad x_1=x_0+0.5x_0^3+\xi_1,\quad x_{t+1}=x_{t}+0.5(x_t-x_{t-1})^3+\xi_t,\label{eqsyn2}}
\end{align}
where $\{\xi_t\}$ is i.i.d with $\xi_t\sim\mcl{N}(0,0.01)$.
{For Eq.~\eqref{eqsyn2}, to extract the relationship between $\tilde{x}_t$ and
$\tilde{x}_{t-p+1},\ldots,\tilde{x}_{t-1}$, we set $x_t$ in dynamical system~\eqref{model} as $x_t:=[y_t,\ldots,y_{t-p+1}]$
for random variable $y_t$ at $t$.}
Using the synthetic data, $K$ was first estimated, then 
the empirical abnormalities $a_{t,S,N}$ were computed using all time-series data with $t=1601$, $N=50,75,100$ and $S=1,\ldots,12$.
We chose time points $t=1601$ for evaluation because the estimation of $K$ requires $\{\tilde{x}_0,\ldots, \tilde{x}_{N\times (S+1)}\}$ and $1601>N\times (S+1)$ for all $N=50,75,100$ and  $S=1,\ldots,12$.
The Gaussian kernel was used, and $\gamma=1+1\iu$, where $\iu$ denotes the imaginary unit, was set for the shift-invert Arnoldi method.
{Theoretically, Perron-Frobenius operators are well-defined for any $c_0$-universal kernel.
Thus, any $c_0$-universal kernel is available for our methods.
Therefore, we chose the Gaussian kernel since it is a typical example of $c_0$-universal kernels.
}

For evaluating the behavior of each method along with $S$, the values
\begin{equation*}
\vert a_{t,S,N}-a_{t,S-1,N}\vert,
\end{equation*}
for $S=2,\ldots,12$ were computed with all time-series data, then the averages of all the time-series data were computed.

The results are shown in Figure~\ref{fig2}. 
If $K$ is bounded, and if the Krylov subspace $\mcl{V}(K,\Phi(\mu_0))$ converges to $\hil_k$, that is, $Q_SQ_S^*$ converges strongly to the identity map in $\hil_k$ as $S\to\infty$,
then by Proposition~\ref{prop:conv2}, $a_{t,S}$ computed with the Arnoldi method converges to $\Vert \phi(x_t)-K\phi(x_{t-1})\Vert_k/\Vert K\phi(x_{t-1})\Vert_k$.
Therefore, in this case, $\vert a_{t,S,N}-a_{t,S-1,N}\vert$, the difference between the empirical abnormality with $S$ and that with $S-1$, becomes smaller as $S$ grows.
{Perron-Frobenius operators of systems without noise associated with the Gaussian kernel are shown to be bounded if and only if the system is linear~\citep{iispre}.
Thus, for linear dynamical system~\eqref{eqsyn0} with small noise, the value $\vert a_{t,S,N}-a_{t,S-1,N}\vert$ computed with the Arnoldi method becomes smaller as $S$ grows in the case of $N=50,75$.
However, those for nonlinear dynamical systems~\eqref{eqsyn} and ~\eqref{eqsyn2} do not seem to decrease even if $S$ grows.
This is due to the unboundedness of $K$.
In addition, for dynamical system~\eqref{eqsyn0} and $N=100$, the value $\vert a_{t,S,N}-a_{t,S-1,N}\vert$ computed with the Arnoldi method also does not seem to decrease even if $S$ grows.
This would be because larger $N$ causes numerical instabilities as we mentioned in Remark~\ref{rem:choice_N}.
}
Meanwhile, we can see $\vert a_{t,S,N}-a_{t,S-1,N}\vert$ computed with the shift-invert Arnoldi method decreases as $S$ grows for all three dynamical systems.
This is because $a_{t,S}$ computed with the shift-invert Arnoldi method converges to $\Vert \phi(x_t)-K\phi(x_{t-1})\Vert_k/\Vert K\phi(x_{t-1})\Vert_k$ even if $K$ is unbounded.
The result indicates that the shift-invert Arnoldi method counteracts the unboundedness of Perron-Frobenius operators.
%
\begin{figure}[t]
\begin{center}
\includegraphics[scale=0.39]{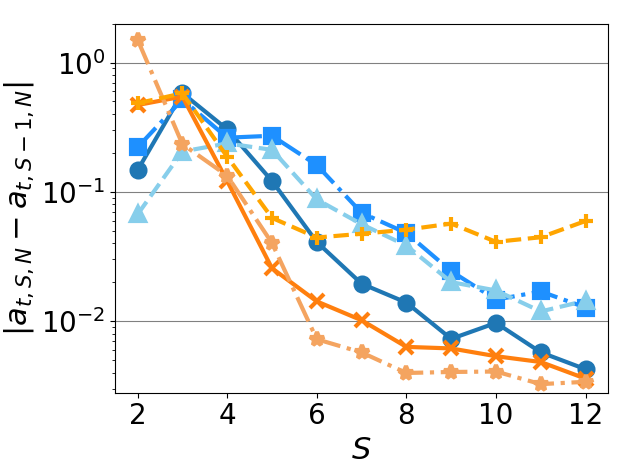}\quad
\includegraphics[scale=0.39]{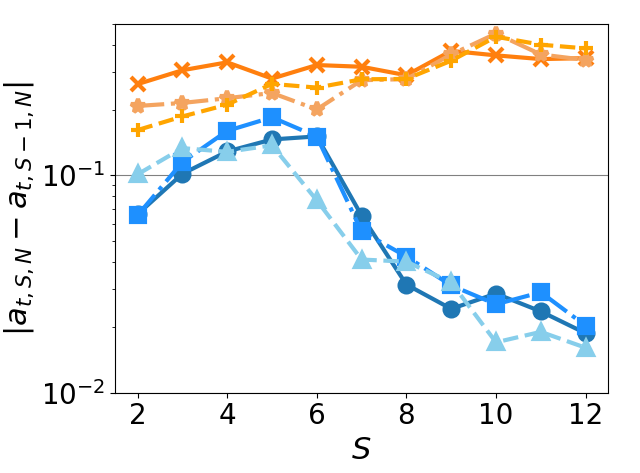}
\includegraphics[scale=0.39]{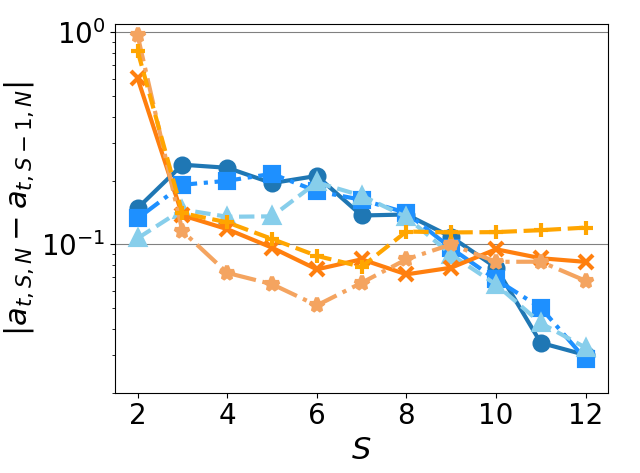}
\includegraphics[scale=0.39]{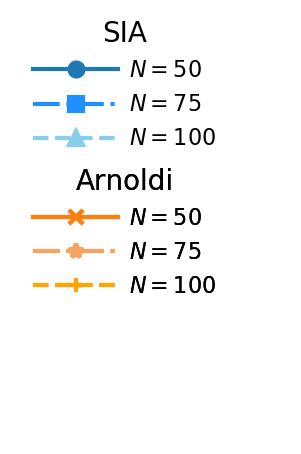}
\caption{{Convergence of the empirical abnormality along $S$ with the synthetic data generated by Eq.~\eqref{eqsyn0} (top left), Eq.~\eqref{eqsyn} (top right), and Eq.~\eqref{eqsyn2} (bottom center).}}\label{fig2}
\end{center}
\end{figure}
\begin{figure}[t]
\begin{center}
\subfigure[$p=15$]{
\includegraphics[scale=0.4]{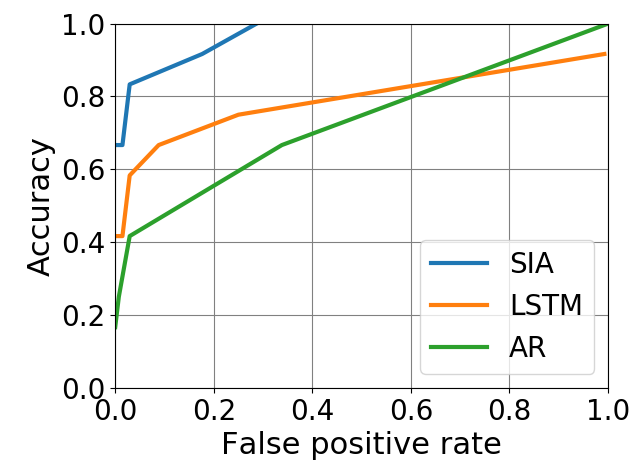}}\hspace{.1cm}
\subfigure[$p=30$]{
\includegraphics[scale=0.4]{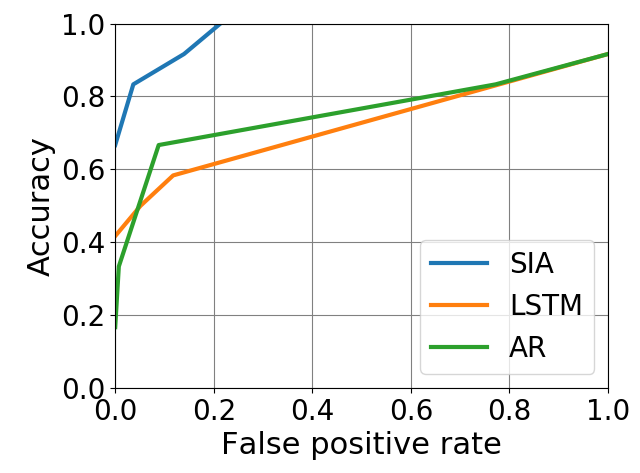}}
 \caption{Accuracy versus false positive rate in anomaly-detection experiments using our shift-invert Arnoldi method and the existing methods: LSTM and AR model}\label{fig6}
\end{center}
\end{figure}
%
\subsection{Anomaly detection with real-world data}\label{secnum2}
We show the empirical results for our shift-invert Arnoldi method in anomaly detection with real-world healthcare data.
We used electrocardiogram (ECG) data~\citep{keoph05}.\footnote{Available at ` \url{http://www.cs.ucr.edu/~eamonn/discords/} '.\\
We used `chfdb{\_}chf01{\_}275.txt', `chfdb{\_}chf13{\_}45590.txt' and `mitdbx{\_}mitdbx{\_}108.txt' in the experiment.}
ECGs are time-series of the electrical
potential between two points on the surface of the body caused by a beating heart.

The graphs in Figure~\ref{fig6} show the accuracy versus the false positive rate for these datasets. 
We first computed $\tilde{\mathbf{K}}_{S,N}$ with $S=10$, $N=40$ 
then computed the empirical abnormality $\hat{a}_{t,S,N}$ for each $t$ with the shift-invert Arnoldi method.
The Laplacian kernel and $\gamma=1.25$ were used.
To extract the relationship between $\tilde{x}_t$ and
$\tilde{x}_{t-p+1},\ldots,\tilde{x}_{t-1}$, we set $x_t$ in dynamical system~\eqref{model} as $x_t:=[y_t,\ldots,y_{t-p+1}]$
for random variable $y_t$ at $t$. In this example, $p$ was set as $p=15,30$,
$\tilde{\mathbf{K}}_{S,N}$ was computed using the data
$\{\tilde{x}_0,\ldots,\tilde{x}_{399}\}$, and
the empirical abnormalities $\hat{a}_{t,S,N}$ at $t=430,431,\ldots$ were computed.
Also, the results obtained using long short-term memory (LSTM)~\citep{pankaj15} and autoregressive (AR) model~\citep{takeuchi06} were evaluated for comparison.
LSTM with $15,30$ time-series and $10$ neurons, with the tanh activation function and the AR model 
$x_{t+1}=\sum_{i=0}^{p-1}c_ix_{t-i}+\xi_t$ with $p=15,30$ 
were used.
The datasets included 12 abnormal parts.
As can be seen, our shift-invert Arnoldi method achieved higher accuracy than LSTM and the AR model while maintaining a low false positive rate for these datasets from complex systems.

\section{Conclusion and Future Work}
\label{sec:conclusion}
In this paper, we addressed a transfer operator to deal with nonlinear dynamical systems with random noise, and developed novel Krylov subspace methods for estimating the transfer operators from finite data. 
For this purpose, we first considered the Perron-Frobenius operators with kernel-mean embeddings for such systems. 
As for the estimation, we extended the Arnoldi method so that it can be applied to the current case. 
Then, we developed the shift-invert Arnoldi method to avoid the problem of the unboundedness of estimated operators because transfer operators on nonlinear systems are not necessarily bounded. 
We also considered an approach of evaluating the prediction accuracy by estimated operators on the basis of the maximum mean discrepancy. 
Finally, we empirically investigated the performance of our methods using synthetic and real-world healthcare data.

In Subsection~\ref{secnum1}, we considered the empirical abnormality which is defined by the prediction error in an RKHS, and showed the convergence of the proposed method as a Krylov subspace method empirically.
As one of our future works, we will address the application of the proposed method to forecasting problems.
If a Perron-Frobenius operator has an eigenvalue whose absolute value is $1$, the corresponding eigenvector describes the time-invariant component of the dynamics.
This fact may be useful for forecasting long-term behaviors of dynamical systems.
Thus, it is expected to be meaningful to consider the eigenvectors of the proposed estimated operator corresponding to eigenvalue $\lambda$ satisfying $\vert\lambda\vert=1$.


\section*{Acknowledgments}
This work was partially supported by JST CREST Grant Number JPMJCR1913.

\appendix




\section{Proofs}\label{ap:pf_kme}

\subsection*{Proof of Lemma~\ref{le:transfer}}
Since $\xi_t$ with $t\in \mathbb{Z}_{\ge 0}$ are i.i.d. and independent of $x_t$, the following identities are derived:
 \begin{align*}
  &K\Phi({x_t}_*P)=\Phi({\beta_t}_*({x_t}_*P\otimes P))
=\int_{\omega\in\varOmega}\int_{x\in\mcl{X}}\phi(\beta_t(x,\omega))\;d{x_t}_*P(x)\;dP(\omega)\nn\\ 
 &\quad=\int_{\omega\in{\varOmega}}\int_{\eta\in{\varOmega}}\phi(h(x_t(\eta))+\xi_t(\omega))\;dP(\eta)\;dP(\omega)
 =\int_{\omega\in{\varOmega}}\phi(h(x_t(\omega))+\xi_t(\omega))\;dP(\omega)\\
 &\quad=\int_{x\in\mcl{X}}\phi(x)\;d(h(x_t)+\xi_t)_*P(x)
=\Phi((h(x_t)+\xi_t)_*P)
=\Phi({x_{t+1}}_*P),\nn 
 \end{align*}
which completes the proof of the lemma.

\subsection*{Proof of Lemma~\ref{le:kme}}
The linearity of $\Phi$ is verified by the definition of $\Phi$.
Next let $\{\mu_N\}_{N=1}^{\infty}$ be a sequence in
$\mcl{M}(\mcl{X})$ such that $\mu=\lim_{N\to\infty}\mu_N$ 
weakly. Then since $k$ is bounded and continuous, the following relations hold:
\begin{align*}
&\Vert\Phi(\mu_N)-\Phi(\mu)\Vert_{k}^2
=\blacket{\Phi\left(\mu_N\right),\Phi\left(\mu_N\right)}_{k}
 -2\Re\blacket{\Phi\left(\mu_N\right),\Phi\left(\mu\right)}_{k}
+\blacket{\Phi\left(\mu\right),\Phi\left(\mu\right)}_{k}\nn\\
 &\qquad=\int_{y\in\mcl{X}}\int_{x\in\mcl{X}}k(x,y)\;d\mu_N(x)\;d\mu_N(y)
-2\Re\int_{y\in\mcl{X}}\int_{x\in\mcl{X}}k(x,y)\;d\mu(x)\;d\mu_N(y)\\
&\phantom{\qquad=\int_{y\in\mcl{X}}\int_{x\in\mcl{X}}k(x,y)\;d\mu_N(x)\;d\mu_N(y)}+\int_{y\in\mcl{X}}\int_{x\in\mcl{X}}k(x,y)\;d\mu(x)\;d\mu(y)\nn\\
 &\qquad\longrightarrow0,
\end{align*}
as $N\rightarrow \infty$, where $\Re z$ for $z\in\mathbb{C}$ is the real part of $z$.
This implies
$\lim_{N\to\infty}\Phi(\mu_N)=\Phi(\mu)$ in $\mathcal{H}_k$. This completes the proof of the lemma.

\subsection*{Proof of Lemma~\ref{le:pf-operator}}
Since each $\xi_t$ for $t\in \mathbb{Z}_{\ge 0}$ is i.i.d. and $K\Phi(\mu)$ is represented as 
$K\Phi(\mu)=\int_{y\in\mcl{X}}\int_{x\in\mcl{X}}\phi(h(x)+y)\;d\mu(x)\;d{\xi_t}_*(y)$,
$K$ does not depend on $t$.

In addition, the identity $\Phi(\delta_x)=\phi(x)$ holds for any $x\in \mathcal{X}$, where $\delta_x$ is the Dirac measure centered at $x\in\mathcal{X}$.
Thus, the inclusion $\opn{Span}\{\phi(x)\mid x\in\mcl{X}\}\subseteq\Phi(\mcl{M}(\mcl{X}))$ holds, which implies $\Phi(\mcl{M}(\mcl{X}))$ is dense in $\hil_k$.
Moreover, according to \cite{sriperumbudur11}, $\Phi$ is injective for $c_0$-universal kernel $k$.
Therefore, the well-definedness of Perron-Frobenius operator $K$, defined as Eq.~\eqref{def} is verified.

Concerning the linearity of $K$, let $c_1,c_2\in\mathbb{C}$ and
$\mu,\nu\in\mcl{M}(\mcl{X})$.
By the linearity of $\Phi$ and the definition of $K$, the following identities hold:
\begin{align*}
 &K\left(c_1\Phi(\mu)+c_2\Phi(\nu)\right)=K\Phi(c_1\mu+c_2\nu)
=\Phi({\beta_t}_*((c_1\mu+c_2\nu)\otimes P))\\
 &\qquad=\Phi(c_1{\beta_t}_*(\mu\otimes P)+c_2{\beta_t}_*(\nu\otimes P))
=c_1\Phi({\beta_t}_*(\mu\otimes P))+c_2\Phi({\beta_t}_*(\nu\otimes P))\\
&\qquad=c_1K\Phi(\mu)+c_2K\Phi(\nu),
\end{align*}
which completes the proof of the lemma.

\subsection*{Proof of Proposition~\ref{prop10}}\label{ap10}
Let $P_{\mid {x_t=x}}$ be a probability measure on $(\varOmega,\mcl{F})$ satisfying $P_{\mid {x_t=x}}(B)=P(B\mid \{x_t=x\})$ for $B\in\mcl{F}$.
Since $p_t$ is the probability density function of $x_t$, the identity $\int_{x\in B}\;d{x_t}_*P(x)=\int_{x\in B}p_t(x)\;d\mu(x)$ holds for any $B\in\mcl{B}$.
Moreover, by the definitions of $p$ and $P_{\mid {x_t=x}}$, the equality $\int_{y\in B}\;d({x_{t+1}}_*P_{\mid {x_t=x}})(y)=\int_{y\in B}p(y\mid x)\;d\mu(y)$ holds for any $B\in\mcl{B}$.
Thus, the following identities are derived:
\begin{align*}
\tilde{\mathscr{K}}_{\opn{RKHS}}\mcl{E}p_t
&=\int_{x\in\mcl{X}}\int_{y\in\mcl{X}}\phi(y)p(y\mid x)p_t(x)\;d\mu(y)\;d\mu(x)\\
&=\int_{x\in\mcl{X}}\int_{y\in\mcl{X}}\phi(y)\;d\left({x_{t+1}}_*P_{\mid {x_t=x}}\right)(y)\;d{x_t}_*P(x).
\end{align*}
Since $x_{t+1}=h(x_t)+\xi_t$, and $x_t$ and $\xi_t$ are independent, the following identities hold for $B\in\mcl{B}$:
\begin{align*}
&\int_{y\in B} \;d\left({x_{t+1}}_*P_{\mid {x_t=x}}\right)(y)
=P\left(\left\{x_{t+1}\in B\right\}\mid \{x_t=x\}\right)
=P\left(\left\{h(x_{t})+\xi_t\in B\right\}\mid \{x_t=x\}\right)\\
&\qquad=\frac{P\left(\left\{h(x)+\xi_t\in B\right\}\bigcap\{x_t=x\}\right)}{P(\{x_t=x\})}
=\frac{P\left(\left\{\xi_t\in B-h(x)\right\}\bigcap\{x_t=x\}\right)}{P(\{x_t=x\})}\\
&\qquad=\frac{P\left(\left\{\xi_t\in B-h(x)\right\}\right)P\left(\{x_t=x\}\right)}{P(\{x_t=x\})}
=P\left(\left\{h(x)+\xi_t\in B\right\}\right)\\
&\qquad=\int_{y\in B} \;d\left({(h(x)+\xi_t)}_*P\right)(y),
\end{align*}
where $B-h(x)$ denotes the set $\{y=z-h(x)\mid z\in B\}$.
Therefore, by the definition of $\beta_t$, the following identities are derived:
\begin{align*}
&\int_{x\in\mcl{X}}\int_{y\in\mcl{X}}\phi(y)\;d\left({x_{t+1}}_*P_{\mid {x_t=x}}\right)(y)\;d{x_t}_*P(x)\\
&\qquad=\int_{x\in\mcl{X}}\int_{y\in\mcl{X}}\phi(y)\;d\left({(h(x)+\xi_t)}_*P\right)(y)\;d{x_t}_*P(x)\\
&\qquad=\int_{x\in\mcl{X}}\phi(x)\;d{\beta_t}_*({x_t}_*P\otimes P)(x)=\Phi\left({\beta_t}_*({x_t}_*P\otimes P\right)).
\end{align*}
By the definition of $K$, the above identities imply $\tilde{\mathscr{K}}_{\opn{RKHS}}\mcl{E}p_t=K\Phi((x_t)_*P)$, which completes the proof of the proposition.

\subsection*{Proof of Proposition~\ref{prop4}}\label{ap6}
By the definition of $K$, the following identities are derived {for $\mu\in\mcl{M}(\mcl{X})$}:
\begin{align*}
{K\Phi(\mu)=\Phi({\beta_t}_*(\mu\otimes P))=\int_{\omega\in\varOmega}\int_{x\in\mcl{X}}\phi(\pi(t,\omega,x))\;d\mu(x)dP(\omega).}
\end{align*}
Let $g\in{\Phi(\mcl{M}(\mcl{X}))}$. Then $g$ is represented as 
{$g=\Phi(\mu)$ with some $\mu\in\mcl{M}(\mcl{X})$}.
Moreover, since $\phi:\mathcal{X}\rightarrow\mathcal{H}_k$ is the feature map, the reproducing property $\blacket{f,\phi(x)}_k=f(x)$ holds for any ${f\in\mcl{D}(\tilde{\mathscr{K}})\subseteq \hil_k}$.
Therefore, the following identities hold:
\begin{align*}
 &\blackets{\tilde{\mathscr{K}}f,g}_{k}=\blackets{\tilde{\mathscr{K}}f,{\Phi(\mu)}}_{k}
={\int_{x\in\mcl{X}}}\blacket{\int_{\omega\in\varOmega}f(\pi(t,\omega,\cdot))\;dP(\omega),\phi(x)}_{k}{d\mu(x)}\\
&\qquad={\int_{x\in\mcl{X}}}\int_{\omega\in\varOmega}f(\pi(t,\omega,x))\;dP(\omega){d\mu(x)}\\
&\qquad =\blacket{f,\int_{\omega\in\varOmega}{\int_{x\in\mcl{X}}}\phi(\pi(t,\omega,x))\;{d\mu(x)}dP(\omega)}_{k} 
=\blacket{f,Kg}_{k},
\end{align*}
which implies that $\tilde{\mathscr{K}}$ is the adjoint operator of ${K}$. This completes the proof of the proposition.





\subsection*{Proof of Proposition \ref{prop9}}\label{ap9}
The left-hand side of assumption~\eqref{eq8} is transformed into
\begin{align*}
   &\lim_{N\to\infty}\frac1N\sum_{i=0}^{N-1}\int_{\omega\in\varOmega}f(h(\tilde{x}_{t+iS'})+\xi_t(\omega))\;dP(\omega)\\
   &\quad=\int_{\omega\in\varOmega}\int_{x\in\mcl{X}}f(h(x)+\xi_t(\omega))\;d{y_t}_*P\;dP(\omega)\\
   &\quad=\int_{x\in\mcl{X}}f(x)\;d{\left(h(y_t)+\xi_t\right)}_*P(x).
\end{align*}
Regarding the right-hand side of assumption~\eqref{eq8}, since $h(\tilde{x}_{t+iS'})+\xi_{t+iS'}(\omega_0)=\tilde{x}_{t+1+iS'}$, the
assumption \eqref{eq8} is equivalent to assumption~\eqref{takeishi_ergodicity} for $\mu={(h(y_t)+\xi_t)}_*P$ and $\tilde{z}_i=\tilde{x}_{t+1+iS'}$.

\subsection*{Proof of Proposition \ref{prop:krylov_extension}}
If $N=1$, then $\mu_{t,N}$ is represented as $\mu_{t,N}=\delta_{\tilde{x}_t}$. Thus, identity 
$\Phi(\mu_{t,N})=\phi(\tilde{x}_t)$ holds.
This implies that in this case, Krylov subspace  
$\mcl{V}_S(K,\Phi(\mu_0))=\opn{Span}\{\Phi(\mu_0),\ldots,\Phi(\mu_{S-1})\}$ is equivalent to Krylov subspace~\eqref{kawahara_krylov}.


\subsection*{Proof of Proposition~\ref{prop:abnormality}}
Since $\VVert Q_S\VVert=\VVert Q_{S,N}\VVert=1$, the following inequalities hold:
\begin{align}
&\Vert Q_{S,N}\tilde{\mathbf{K}}_{S,N}Q_{S,N}^*\phi(\tilde{x}_{t-1})-Q_S\tilde{\mathbf{K}}_SQ_S^*\phi(\tilde{x}_{t-1})\Vert_k\nn\\
&\qquad\le \Vert Q_{S,N}\tilde{\mathbf{K}}_{S,N}Q_{S,N}^*\phi(\tilde{x}_{t-1})-Q_{S,N}\tilde{\mathbf{K}}_{S,N}Q_{S}^*\phi(\tilde{x}_{t-1})\Vert_k\nn\\
&\qquad\qquad+\Vert Q_{S,N}\tilde{\mathbf{K}}_{S,N}Q_{S}^*\phi(\tilde{x}_{t-1})-Q_{S,N}\tilde{\mathbf{K}}_{S}Q_{S}^*\phi(\tilde{x}_{t-1})\Vert_k\nn\\
&\qquad\qquad+\Vert Q_{S,N}\tilde{\mathbf{K}}_{S}Q_{S}^*\phi(\tilde{x}_{t-1})-Q_{S}\tilde{\mathbf{K}}_{S}Q_{S}^*\phi(\tilde{x}_{t-1})\Vert_k\nn\\
&\qquad\le \VVert \tilde{\mathbf{K}}_{S,N}\VVert\Vert Q_{S,N}^*\phi(\tilde{x}_{t-1})-Q_{S}^*\phi(\tilde{x}_{t-1})\Vert_k\nn\\
&\qquad\qquad+\Vert \tilde{\mathbf{K}}_{S,N}Q_{S}^*\phi(\tilde{x}_{t-1})-\tilde{\mathbf{K}}_{S}Q_{S}^*\phi(\tilde{x}_{t-1})\Vert_k\nn\\
&\qquad\qquad+\Vert Q_{S,N}\tilde{\mathbf{K}}_{S}Q_{S}^*\phi(\tilde{x}_{t-1})-Q_{S}\tilde{\mathbf{K}}_{S}Q_{S}^*\phi(\tilde{x}_{t-1})\Vert_k\label{eq40}
\end{align}
Since the elements of $Q_{S,N}^*\phi(\tilde{x}_{t-1})\in\mathbb{C}^{S}$ are composed
of the finite linear combinations of inner products between $\Phi(\mu_{t,N})$ 
and $\phi(\tilde{x}_{t-1})$ in the RKHS, the same discussion as $\mathbf{R}_{S,N}$ and $Q_{S,N}^*\Psi_{1,N}$ in Proposition~\ref{prop1} derives
$Q_{S,N}^*\phi(\tilde{x}_{t-1})\to Q_{S}^*\phi(\tilde{x}_{t-1})$ as $N\to\infty$.
Thus, the first term of Eq.~\ref{eq40} converges to $0$ as $N\to\infty$.
In addition, by Proposition~\ref{prop1}, $\tilde{\mathbf{K}}_{S,N}\to\tilde{\mathbf{K}}_S$ and $Q_{S,N}\to Q_S$ strongly in $\hil_k$ as $N\to\infty$, which implies 
the second and third terms of Eq.~\ref{eq40} also converge to $0$ as $N\to\infty$.
Therefore, $Q_{S,N}\tilde{\mathbf{K}}_{S,N}Q_{S,N}^*\phi(\tilde{x}_{t-1})$ converges to $Q_S\tilde{\mathbf{K}}_SQ_S^*\phi(\tilde{x}_{t-1})$ as $N\to\infty$.
Since the norm $\Vert\cdot\Vert_k$ is continuous, $\Vert Q_{S,N}\tilde{\mathbf{K}}_{S,N}Q_{S,N}^*\phi(\tilde{x}_{t-1})\Vert_k$ and $\Vert \phi(\tilde{x}_{t-1})- Q_{S,N}\tilde{\mathbf{K}}_{S,N}Q_{S,N}^*\phi(\tilde{x}_{t-1})\Vert_k$ converge to $\Vert Q_S\tilde{\mathbf{K}}_SQ_S^*\phi(\tilde{x}_{t-1})\Vert_k$ and $\Vert \phi(\tilde{x}_{t-1})- Q_S\tilde{\mathbf{K}}_SQ_S^*\phi(\tilde{x}_{t-1})\Vert_k$ as $N\to\infty$, respectively.
This implies $a_{t,S,N}\to a_{t,S}$ as $N\to\infty$. 

\section{Computation of QR decomposition of $\Psi_{1,N}$ and $\tilde{\mathbf{K}}_{S,N}$}\label{ap7}
For implementing the Arnoldi method and shift-invert Arnoldi method described in
Section~\ref{sec:krylov}, QR decomposition must be computed.
In this section, we explain the method to compute the QR decomposition.
The orthonormal basis of 
$\opn{Span}\{\Phi(\mu_{0,N}),\ldots,\Phi(\mu_{S-1,N})\}$ for the Arnoldi method or $\opn{Span}\{\sum_{j=0}^i\binom{i}{j}(-1)^{j}\gamma^{i-j}\Phi(\mu_{j})\mid 1\le i\le S\}$ for the shift-invert Arnoldi method, which is denoted as $q_{0,N},\ldots,q_{S-1,N}$, is obtained through 
QR decomposition.
Then, $Q_{S,N}$ is defined as the operator that maps $[y_{0},\ldots,y_{S-1}]\in\mathbb{C}^S$ to $y_0q_{0,N}+\ldots,y_{S-1}q_{S-1,N}$.
The adjoint operator $Q_{S,N}^*$ maps $v\in\hil_k$ to $[\blacket{v,q_{0,N}}_k,\ldots,\blacket{v,q_{S-1,N}}_k]\in\mathbb{C}^S$.

First, the QR decomposition for the Arnoldi method is shown.
For $t=0$, $q_{0,N}$ is set as
$q_{0,N}:=\Phi(\mu_{0,N})/\Vert\Phi(\mu_{0,N})\Vert_k$.
For $t>0$, $q_{t,N}$ is computed using $q_{0,N},\ldots,q_{t-1,N}$ as follows:
\begin{equation}
\begin{aligned}
 \tilde{q}_{t,N}&:=\Phi(\mu_{t,N})-\sum_{i=0}^{t-1}\blacket{\Phi(\mu_{t,N}),q_{i,N}}_kq_{i,N}\\
 q_{t,N}&:=\tilde{q}_{t,N}/\Vert \tilde{q}_{t,N}\Vert_k.
 \end{aligned}\label{eqap1}
\end{equation}
Let the $(i,t)$-element of $\mathbf{R}_{S,N}$ be
$r_{i,t}$, where $r_{i,t}$ is set as $r_{i,t}:=\blacket{\Phi(\mu_{t,N}),q_{i,N}}_k$ for $i<t$, 
$r_{i,t}:=\Vert\tilde{q}_{t,N}\Vert_k$ for $i=t$, $r_{i,t}=0$ for $i>t$.
Then, by Eq.~\eqref{eqap1} and the definition of $\mathbf{R}_{S,N}$, $q_{i,N}$ is represented as
$q_{i,N}=(\Phi(\mu_{i,N})-\sum_{j=0}^{i-1}r_{j,i}q_j)/r_{i,i}$, and $\Psi_{0,N}=Q_{S,N}\mathbf{R}_{S,N}$ holds.
Therefore, by definition of $r_{j,t}\ (j=0,\ldots,i-1)$, $r_{i,t}$ is computed as follows for $i<t$:
\begin{align*}
 r_{i,t}&=\blacket{\Phi(\mu_{t,N}),q_{i,N}}_k
 =\bblacketg{\Phi(\mu_{t,N}),\frac{\Phi(\mu_{i,N})-\sum_{j=0}^{i-1}r_{j,i}q_j}{r_{i,i}}}_k\\
 &=\frac{\blacket{\Phi(\mu_{t,N}),\Phi(\mu_{i,N})}_k-\sum_{j=0}^{i-1}\overline{r_{j,i}}\blacket{\Phi(\mu_{t,N}),q_j}_k}{r_{i,i}}
 =\frac{\blacket{\Phi(\mu_{t,N}),\Phi(\mu_{i,N})}_k-\sum_{j=0}^{i-1}\overline{r_{j,i}}r_{j,t}}{r_{i,i}}.
\end{align*}
Since $\mu_{t,N}=1/N\sum_{j=0}^{N-1}\delta_{\tilde{x}_{t+jS}}$,
$\blacket{\Phi(\mu_{t,N}),\Phi(\mu_{i,N})}_k$ is computed as follows:
\begin{align*}
 &\blacket{\Phi(\mu_{t,N}),\Phi(\mu_{i,N})}_k={\bblacketg{\frac{1}{N}\sum_{j=0}^{N-1}\Phi(\delta_{\tilde{x}_{t+jS}}),\frac{1}{N}\sum_{j=0}^{N-1}\Phi(\delta_{\tilde{x}_{i+jS}})}_k}\\
 &\qquad{=\frac{1}{N^2}\sum_{j,l=0}^{N-1}\blacket{\Phi(\delta_{\tilde{x}_{t+jS}}),\Phi(\delta_{\tilde{x}_{i+lS}})}_k}
 ={\frac{1}{N^2}}\sum_{j,l=0}^{N-1}\blacket{\phi(\tilde{x}_{t+jS}),\phi(\tilde{x}_{i+lS})}_k\\
&\qquad ={\frac{1}{N^2}}\sum_{j,l=0}^{N-1}k(\tilde{x}_{t+jS},\tilde{x}_{i+lS}).
\end{align*}

Similarly, by the definition of $r_{j,t}\ (j=0,\ldots,t-1)$, $r_{t,t}$ is computed, since $\blacket{q_{i,N},q_{j,N}}_k=1$ for $i=j$ and $\blacket{q_{i,N},q_{j,N}}_k=0$ for $i\neq j$ as follows: 
\begin{align*}
r_{t,t}^2&=\blacket{\tilde{q}_{t,N}\tilde{q}_{t,N}}_k\\
&=\bblacketg{\Phi(\mu_{t,N})-\sum_{j=0}^{t-1}\blacket{\Phi(\mu_{t,N}),q_{j,N}}_kq_{j,N}
,\Phi(\mu_{t,N})-\sum_{j=0}^{t-1}\blacket{\Phi(\mu_{t,N}),q_{j,N}}_kq_{j,N}}_k\\
&=\bblacketg{\Phi(\mu_{t,N})-\sum_{j=0}^{t-1}r_{j,t}q_{j,N},\Phi(\mu_{t,N})-\sum_{j=0}^{t-1}r_{j,t}q_{j,N}}_k\\
&=\blackets{\Phi(\mu_{t,N}),\Phi(\mu_{t,N})}_k-2\Re{\bblacketg{\sum_{j=0}^{t-1}r_{j,t}q_{j,N},\Phi(\mu_{t,N})}_k}
+\bblacketg{\sum_{j=0}^{t-1}r_{j,t}q_{j,N},\sum_{j=0}^{t-1}r_{j,t}q_{j,N}}_k\\
&={\frac{1}{N^2}}\sum_{j,l=0}^{N-1}k(\tilde{x}_{t+jS},\tilde{x}_{t+lS})-2\sum_{j=0}^{t-1}r_{j,t}\overline{r_{j,t}}
+\sum_{j=0}^{t-1}r_{j,t}\overline{r_{j,t}}\\
&={\frac{1}{N^2}}\sum_{j,l=0}^{N-1}k(\tilde{x}_{t+jS},\tilde{x}_{t+lS})-\sum_{j=0}^{t-1}\vert r_{j,t}\vert^2,
\end{align*}
where $\Re z$ for $z\in\mathbb{C}$ is the real part of $z$.
The above computations construct $\mathbf{R}_{S,N}$.
Then, since the $(i,t)$ element of $Q_{S,N}^*\Psi_1$ is represented as $\blacket{\Phi(\mu_{t+1,N}),q_{i,N}}_k$, $Q_{S,N}^*\Psi_{1,N}$ is computed in the same manner as $\mathbf{R}_{S,N}$.
The $\tilde{\mathbf{K}}_{S,N}$ is obtained by $Q_{S,N}^*\Psi_1\mathbf{R}_{S,N}^{-1}$.

For the shift-invert Arnoldi method, by Eq.~\eqref{eq30}, the projection space is represented as $\opn{Span}\{w_{1,N},\ldots,w_{S,N}\}$.
Thus, $\blacket{\Phi(\mu_{t,N}),\Phi(\mu_{i,N})}_k$ is replaced with 
$\blacket{w_{t+1,N},w_{i+1,N}}_k=\sum_{l=0}^{i+1}\sum_{j=0}^{t+1}\binom{i+1}{l}(-1)^{j+l}\binom{t+1}{j}\overline{\gamma}^{1+i-l}\gamma^{1+t-j}\blacket{\Phi(\mu_{j,N}),\Phi(\mu_{l,N})}_k$.

\section{Pseudo-codes of Arnoldi and shift-invert Arnoldi methods}\label{ap8}
Let $\mathbf{R}_{S:T}$ be the matrix composed of $r_{i,t}\ (S\le t\le T,\ 0\le i\le S-1)$.
The pseudo-codes for computing $\tilde{\mathbf{K}}_S$ with the Arnoldi method and  shift-invert Arnoldi method are shown in Algorithms~\ref{al:arnoldi} and \ref{al:sia}, respectively.
\begin{algorithm}[t]
\caption{Arnoldi method for Perron-Frobenius operator $K$ in an RKHS}\label{al:arnoldi}
 \begin{algorithmic}[1]
 \Require $S,N\in\mathbb{N},\ \{\tilde{x}_0,\ldots,\tilde{x}_{NS-1}\}$
 \Ensure $\tilde{\mathbf{K}}_{S,N}$
  \For{$t=0,\ldots,S$}
  \For{$i=0,\ldots,S-1$}
  \If{$i<t$}
  \State $r_{i,t}=({1/{N^2}}\sum_{j,l=0}^{N-1}k(\tilde{x}_{t+lS},\tilde{x}_{i+jS})-\sum_{j=0}^{i-1}\overline{r_{j,i}}r_{j,t})/r_{i,i}$
  \ElsIf{$i=t$}
  \State $r_{t,t}=\sqrt{{1/{N^2}}\sum_{j,l=0}^{N-1}k(\tilde{x}_{t+jS},\tilde{x}_{t+lS})-\sum_{j=0}^{t-1}\vert r_{j,t}\vert^2}$
  \Else
  \State $r_{i,t}=0$
  \EndIf
  \EndFor
  \EndFor
  \State $\tilde{\mathbf{K}}_{S,N}=\mathbf{R}_{1:S}\mathbf{R}_{0:S-1}^{-1}$
 \end{algorithmic}
\end{algorithm}
\begin{algorithm}[t]
\caption{Shift-invert Arnoldi method for Perron-Frobenius operator $K$ in an RKHS}\label{al:sia}
 \begin{algorithmic}[1]
  \Require $S,N\in\mathbb{N},\ \gamma\notin\varLambda(K),\ \{\tilde{x}_0,\ldots,\tilde{x}_{NS-1}\}$
 \Ensure $\tilde{\mathbf{K}}_{S,N}$
  \For{$t=0,\ldots,S-1$}
  \For{$i=0,\ldots,t$}
  \State $g_{i,t}={1/{N^2}}\sum_{j,l=0}^{N-1}k(\tilde{x}_{i+jS},\tilde{x}_{t+lS})$
  \EndFor
  \EndFor
   \For{$t=0,\ldots,S-1$}
  \For{$i=0,\ldots,S-1$}
    \If{$i<t$}
  \State $r_{i,t}=(\sum_{l=0}^{i+1}\sum_{j=0}^{t+1}\binom{i+1}{l}\binom{t+1}{j}(-1)^{j+l}\overline{\gamma}^{1+i-l}\gamma^{1+t-j}g_{j,l}-\sum_{j=0}^{i-1}\overline{r_{j,i}}r_{j,t})/r_{i,i}$
  \ElsIf{$i=t$}
  \State $r_{t,t}=\sqrt{\sum_{j,l=0}^{t+1}\binom{t+1}{j}\binom{t+1}{l}(-1)^{j+l}\gamma^{1+t-j}\overline{\gamma}^{1-t-l}g_{j,l}-\sum_{j=0}^{t-1}\vert r_{j,t}\vert^2}$
  \Else
  \State $r_{i,t}=0$
  \EndIf
  \EndFor
  \EndFor
  \For{$i=0,\ldots,S-1$}
  \State $\tilde{r}_i=(\sum_{l=0}^{i+1}\binom{i+1}{l}(-1)^{l}\overline{\gamma}^{i+1-l}g_{0,l}-\sum_{j=0}^{i-1}\overline{r_{j,i}}\tilde{r}_j)/r_{i,i}$
  \EndFor
  \State $\tilde{\mathbf{K}}_{S,N}=[\tilde{r}\ \mathbf{R}_{0:S-2}]\mathbf{R}_{0:S-1}^{-1}$
 \end{algorithmic}

\end{algorithm}


\begin{thebibliography}{29}
\providecommand{\natexlab}[1]{#1}
\providecommand{\url}[1]{\texttt{#1}}
\expandafter\ifx\csname urlstyle\endcsname\relax
  \providecommand{\doi}[1]{doi: #1}\else
  \providecommand{\doi}{doi: \begingroup \urlstyle{rm}\Url}\fi

\bibitem[Budi\v{s}i\'{c} et~al.(2012)Budi\v{s}i\'{c}, Mohr, and
  Mezi\'{c}]{mezic12}
M.~Budi\v{s}i\'{c}, R.~Mohr, and I.~Mezi\'{c}.
\newblock Applied {K}oopmanism.
\newblock \emph{Chaos (Woodbury, N.Y.)}, 22:\penalty0 047510, 2012.

\bibitem[{\v{C}rnjari\'{c}-\v{Z}ic} et~al.(2019){\v{C}rnjari\'{c}-\v{Z}ic},
  Ma\'{c}e\v{s}i\'{c}, and Mezi\'{c}]{mezic17}
N.~{\v{C}rnjari\'{c}-\v{Z}ic}, S.~Ma\'{c}e\v{s}i\'{c}, and I.~Mezi\'{c}.
\newblock Koopman operator spectrum for random dynamical systems.
\newblock \emph{Journal of Nonlinear Science}, 2019.

\bibitem[Crouzeix and Palencia(2017)]{cro17}
M.~Crouzeix and C.~Palencia.
\newblock The numerical range is a $(1+\sqrt{2})$-spectral set.
\newblock \emph{SIAM Journal on Matrix Analysis and Applications}, 38\penalty0
  (2):\penalty0 649--655, 2017.

\bibitem[Gallopoulos and Saad(1992)]{gallopoulos92}
E.~Gallopoulos and Y.~Saad.
\newblock Efficient solution of parabolic equations by {K}rylov approximation
  methods.
\newblock \emph{SIAM Journal on Scientific and Statistical Computing},
  13\penalty0 (5):\penalty0 1236--1264, 1992.

\bibitem[G{\"{o}}ckler(2014)]{gockler14}
T.~G{\"{o}}ckler.
\newblock \emph{Rational {K}rylov Subspace Methods for '$\phi$'-functions in
  Exponential Integrators}.
\newblock PhD thesis, Karlsruher Instituts f{\"{u}}r Technologie, 2014.

\bibitem[Gretton et~al.(2012)Gretton, Borgwardt, Rasch, Sch\"{o}lkopf, and
  Smola]{gretton12}
A.~Gretton, K.~M. Borgwardt, M.~J. Rasch, B.~Sch\"{o}lkopf, and A.~Smola.
\newblock A kernel two-sample test.
\newblock \emph{Journal of Machine Learning Research}, 13\penalty0
  (25):\penalty0 723--773, 2012.

\bibitem[Grimm(2012)]{grimm12}
V.~Grimm.
\newblock Resolvent {K}rylov subspace approximation to operator functions.
\newblock \emph{BIT Numerical Mathematics}, 52:\penalty0 639--659, 2012.

\bibitem[G{\"{u}}ttel(2010)]{guttel10}
S.~G{\"{u}}ttel.
\newblock \emph{Rational {K}rylov Methods for Operator Functions}.
\newblock PhD thesis, Techniche Universit{\"{a}}t Bergakademie Freiberg, 2010.

\bibitem[Hashimoto and Nodera(2019)]{hashimoto_jjiam}
Y.~Hashimoto and T.~Nodera.
\newblock Shift-invert rational {K}rylov method for an operator $\phi$-function
  of an unbounded linear operator.
\newblock \emph{Japan Journal of Industrial and Applied Mathematics},
  36\penalty0 (2):\penalty0 421--433, 2019.

\bibitem[Ikeda et~al.(2019)Ikeda, Ishikawa, and Sawano]{iispre}
M.~Ikeda, I.~Ishikawa, and Y.~Sawano.
\newblock Composition operators on reproducing kernel {H}ilbert spaces with
  analytic positive definite functions.
\newblock \emph{arXiv:1911.11992}, 2019.

\bibitem[Ishikawa et~al.(2018)Ishikawa, Fujii, Ikeda, Hashimoto, and
  Kawahara]{ishikawa18}
I.~Ishikawa, K.~Fujii, M.~Ikeda, Y.~Hashimoto, and Y.~Kawahara.
\newblock Metric on nonlinear dynamical systems with {P}erron-{F}robenius
  operators.
\newblock In \emph{Advances in Neural Information Processing Systems 31}, pages
  2856--2866, 2018.

\bibitem[Kawahara(2016)]{kawahara16}
Y.~Kawahara.
\newblock Dynamic mode decomposition with reproducing kernels for {K}oopman
  spectral analysis.
\newblock In \emph{Advances in Neural Information Processing Systems 29}, pages
  911--919, 2016.

\bibitem[{Keogh} et~al.(2005){Keogh}, {Lin}, and {Fu}]{keoph05}
E.~{Keogh}, J.~{Lin}, and A.~{Fu}.
\newblock Hot sax: efficiently finding the most unusual time series
  subsequence.
\newblock In \emph{Fifth IEEE International Conference on Data Mining}, 2005.

\bibitem[Klus et~al.(2020)Klus, Schuster, and Muandet]{klus17}
S.~Klus, I.~Schuster, and K.~Muandet.
\newblock Eigendecompositions of transfer operators in reproducing kernel
  {H}ilbert spaces.
\newblock \emph{Journal of Nonlinear Science}, 30:\penalty0 283--315, 2020.

\bibitem[Koopman(1931)]{koopman31}
B.~O. Koopman.
\newblock Hamiltonian systems and transformation in {H}ilbert space.
\newblock \emph{Proceedings of the National Academy of Sciences}, 17\penalty0
  (5):\penalty0 315--318, 1931.

\bibitem[Krylov(1931)]{krylov31}
A.~N. Krylov.
\newblock On the numerical solution of the equation by which in technical
  questions frequencies of small oscillations of material systems are
  determined.
\newblock \emph{Izvestija AN SSSR}, 7\penalty0 (4):\penalty0 491--539, 1931.
\newblock (in Russian).

\bibitem[Kubrusly(2012)]{kubrusly12}
C.~S. Kubrusly.
\newblock \emph{Spectral Theory of Operators on Hilbert Spaces}.
\newblock Birkh{\"{a}}user Basel, 2012.

\bibitem[Kutz(2013)]{kutz13}
J.~N. Kutz.
\newblock \emph{Data-Driven Modeling \& Scientific Computation: Methods for
  Complex Systems \& Big Data}.
\newblock Oxford University Press, 2013.

\bibitem[Lusch et~al.(2018)Lusch, N.~Kutz, and L.~Brunton]{lusch17}
B.~Lusch, J.~N.~Kutz, and S.~L.~Brunton.
\newblock Deep learning for universal linear embeddings of nonlinear dynamics.
\newblock \emph{Nature Communications}, 9:\penalty0 4950, 2018.

\bibitem[Malhotra et~al.(2015)Malhotra, Vig, Shroff, and Agarwal]{pankaj15}
P.~Malhotra, L.~Vig, G.~Shroff, and P.~Agarwal.
\newblock Long short term memory networks for anomaly detection in time series.
\newblock In \emph{European Symposium on Artificial Neural Networks,
  Computational Intelligence and Machine Learning}, pages 89--94, 2015.

\bibitem[McIntosh(1978)]{mcintosh78}
A.~McIntosh.
\newblock The {T}oeplitz-{H}ausdorff theorem and ellipticity conditions.
\newblock \emph{The American Mathematical Monthly}, 85\penalty0 (6):\penalty0
  475--477, 1978.

\bibitem[Moret and Novati(2004)]{moret04}
I.~Moret and P.~Novati.
\newblock {RD}-rational approximations of the matrix exponential.
\newblock \emph{BIT Numerical Mathematics}, 44:\penalty0 595--615, 2004.

\bibitem[Muandet et~al.(2017)Muandet, Fukumizu, Sriperumbudur, and
  Sch\"{o}lkopf]{kernelmean}
K.~Muandet, K.~Fukumizu, B.~K. Sriperumbudur, and B.~Sch\"{o}lkopf.
\newblock Kernel mean embedding of distributions: a review and beyond.
\newblock \emph{Foundations and Trends in Machine Learning}, 10\penalty0
  (1--2), 2017.

\bibitem[R~Hestenes and Stiefel(1952)]{hestens52}
M.~R~Hestenes and E.~Stiefel.
\newblock Methods of conjugate gradients for solving linear systems.
\newblock \emph{Journal of Research of the National Bureau of Standards},
  49\penalty0 (6):\penalty0 409--436, 1952.

\bibitem[Saad and Schultz(1986)]{saad83}
Y.~Saad and M.~H. Schultz.
\newblock {GMRES:} a generalized minimal residual algorithm for solving
  nonsymmetric linear systems.
\newblock \emph{SIAM Journal on Scientific and Statistical Computing},
  7\penalty0 (3):\penalty0 856--869, 1986.

\bibitem[Sriperumbudur et~al.(2011)Sriperumbudur, Fukumizu, and
  Lanckriet]{sriperumbudur11}
B.~K. Sriperumbudur, K.~Fukumizu, and G.~R.~G. Lanckriet.
\newblock Universality, characteristic kernels and {RKHS} embedding of
  measures.
\newblock \emph{Journal of Machine Learning Research}, 12\penalty0
  (70):\penalty0 2389--2410, 2011.

\bibitem[Takeishi et~al.(2017{\natexlab{a}})Takeishi, Kawahara, and
  Yairi]{takeishi17}
N.~Takeishi, Y.~Kawahara, and T.~Yairi.
\newblock Subspace dynamic mode decomposition for stochastic {K}oopman
  analysis.
\newblock \emph{Physical Review E}, 96:\penalty0 033310, 2017{\natexlab{a}}.

\bibitem[Takeishi et~al.(2017{\natexlab{b}})Takeishi, Kawahara, and
  Yairi]{takeishi17-2}
N.~Takeishi, Y.~Kawahara, and T.~Yairi.
\newblock Learning {K}oopman invariant subspaces for dynamic mode
  decomposition.
\newblock In \emph{Advances in Neural Information Processing Systems 30}, pages
  1130--1140, 2017{\natexlab{b}}.

\bibitem[{Takeuchi} and {Yamanishi}(2006)]{takeuchi06}
J.~{Takeuchi} and K.~{Yamanishi}.
\newblock A unifying framework for detecting outliers and change points from
  time series.
\newblock \emph{IEEE Transactions on Knowledge and Data Engineering},
  18\penalty0 (4):\penalty0 482--492, 2006.

\end{thebibliography}

\end{document}